\documentclass{article}

\usepackage{microtype}

\usepackage{amsmath,amsfonts,bm}

\def\eqref#1{equation~\ref{#1}}

\def\1{\bm{1}}

\DeclareMathAlphabet{\mathsfit}{\encodingdefault}{\sfdefault}{m}{sl}
\SetMathAlphabet{\mathsfit}{bold}{\encodingdefault}{\sfdefault}{bx}{n}

\usepackage[hidelinks]{hyperref}
\usepackage{url}
\usepackage{graphicx}
\usepackage{amsthm}
\usepackage{booktabs}
\usepackage[capitalize, nameinlink, poorman, noabbrev]{cleveref}
\usepackage{threeparttable} %
\usepackage{caption}
\usepackage{subcaption}
\usepackage{wrapfig}
\usepackage{textcomp}
\usepackage{placeins}

\usepackage[docdef=atom]{glossaries-extra}
    \setabbreviationstyle[acronym]{long-short}
    \glssetcategoryattribute{acronym}{nohyperfirst}{true}

    \newacronym{sota}{SotA}{State-of-the-Art}

\newacronym{tcn}{TCN}{Temporal Convolutional Network}
\newacronym{eid}{EID}{Exponentially Increasing Dilation Rates}
\newacronym{cd}{CD}{Constant Dilation Rate}
\newacronym{ssm}{SSM}{State-Space Model}
\newacronym{rnn}{RNN}{Recurrent Neural Network}
\newacronym{lstm}{LSTM}{Long Short-Term Memory}
\newacronym{gru}{GRU}{Gated Recurrent Unit}
\newacronym{mgrade}{mGRADE}{\textbf{m}inimal \textbf{G}ated \textbf{R}ecurrent \textbf{A}rchitecture with \textbf{D}elay \textbf{E}mbeddings}
\newacronym{lra}{LRA}{Long-Range Arena}
\newacronym{smnist}{sMNIST}{sequential MNIST}
\newacronym{scifar}{sCIFAR}{sequential CIFAR}
\newacronym{bptt}{BPTT}{Backpropagation through Time}
\newacronym{mingru}{minGRU}{minimal Gated Recurrent Unit}
\newacronym{xlstm}{xLSTM}{Extended Long Short-Term Memory}
\newacronym{s4}{S4}{Structured State Space Model}
\newacronym{h3}{H3}{Hungry Hungry Hippo}
\newacronym{sgconv}{SGConv}{Structured Global Convolution}

\newacronym{hgrn}{HGRN}{Hierarchically Gated Recurrent Network}
\newacronym{qrnn}{QRNN}{Quasi-Recurrent Neural Network}
\newacronym{lru}{LRU}{Linear Recurrent Unit}
\newacronym{lti}{LTI}{linear time-invariant}
\newacronym{dcls}{DCLS}{Dilated Convolutions with Learnable Spacings}
\newacronym{mlp}{MLP}{Multi-layer Perceptron}
\newacronym{mase}{MASE}{Mean Absolute Standardized Error}
\newacronym{imc}{IMC}{In-Memory Computing}
\newacronym{ce}{CE}{Cross-Entropy Loss}
\newacronym{mse}{MSE}{Mean Squared Error}
\newacronym{pc}{PC}{Principal Component}
\newacronym{nno}{NNO}{Nearest Neighbours Overlap}

\newacronym{gilr}{GILR}{Gated Impulse Linear Recurrent}
\newacronym{relu}{ReLU}{Rectified Linear Unit}
\newacronym{ugi}{UGI}{Uniform Gate Initialization}
\newacronym{siso}{SISO}{Single-Input, Single-Output}
\newacronym{mimo}{MIMO}{Multi-Input, Multi-Output}
\newacronym{dplr}{DPLR}{Diagonal Plus Low Rank}
\newacronym{gsc}{GSC}{Google Speech Commands}
\newacronym{shd}{SHD}{Spiking Heidelberg Digits}
\newacronym{snn}{SNN}{Spiking Neural Networks}

\usepackage[obeyFinal]{todonotes}

\crefname{section}{}{}
\Crefname{section}{}{}

\newtheorem{theorem}{Theorem}
\theoremstyle{definition}
\newtheorem{defn}{Definition}
\theoremstyle{remark}

\usepackage[preprint]{icml2026/icml2026}

\usepackage{amsmath}
\usepackage{amssymb}
\usepackage{mathtools}
\usepackage{amsthm}

\icmltitlerunning{mGRADE: Minimal Recurrent Gating Meets Delay Convolutions for Lightweight Sequence Modeling}

\begin{document}

\twocolumn[
  \icmltitle{mGRADE: Minimal Recurrent Gating Meets Delay Convolutions for Lightweight Sequence Modeling}

  \icmlsetsymbol{equal}{*}

  \begin{icmlauthorlist}
    \icmlauthor{Tristan Torchet}{equal,yyy}
    \icmlauthor{Christian Metzner}{equal,yyy}
    \icmlauthor{Karthik Charan Raghunathan}{yyy}
    \icmlauthor{Jimmy Weber}{yyy}
    \icmlauthor{Sebastian Billaudelle}{yyy}
    \icmlauthor{Laura Kriener}{yyy}
    \icmlauthor{Melika Payvand}{yyy}
  \end{icmlauthorlist}

  \icmlaffiliation{yyy}{Institute of Neuroinformatics, University of Zurich and ETH Zurich, Zurich, Switzerland}

  \icmlcorrespondingauthor{Melika Payvand}{\{melika\}@ini.uzh.ch}
  \icmlkeywords{sequence model, recurrent neural networks, efficient architectures, dynamical systems, Long Range Arena}

  \vskip 0.3in
]

\printAffiliationsAndNotice{\icmlEqualContribution}

\begin{abstract}
Multi-timescale sequence modeling relies on capturing both local fast dynamics and global slow context; yet, maintaining these capabilities under the strict memory constraints common to edge devices remains an open challenge.
Current \gls{sota} models with constant memory footprints trade off long-range selectivity and high-precision modeling of fast dynamics.
To overcome this trade-off within a fixed memory budget, we propose \textbf{mGRADE} (\textbf{m}inimally \textbf{G}ated \textbf{R}ecurrent \textbf{A}rchitecture with \textbf{D}elay \textbf{E}mbedding), a hybrid-memory system that introduces inductive biases across timescales by integrating a convolution with learnable temporal spacings with a lightweight gated recurrent component.
We show theoretically that the learnable spacings are equivalent to a delay embedding, enabling parameter-efficient reconstruction of partially-observed fast dynamics, while the gated recurrent component selectively maintains long-range context with minimal memory overhead.
On the challenging Long-Range Arena benchmark and 35-way Google Speech Commands raw audio classification task, mGRADE reduces the memory footprint by up to a factor of 8 compared to other \gls{sota} models, while maintaining competitive performance.

\end{abstract}

\glsresetall
\section{Introduction}
\begin{figure*}[t]
    \centering
    \includegraphics[width=\textwidth]{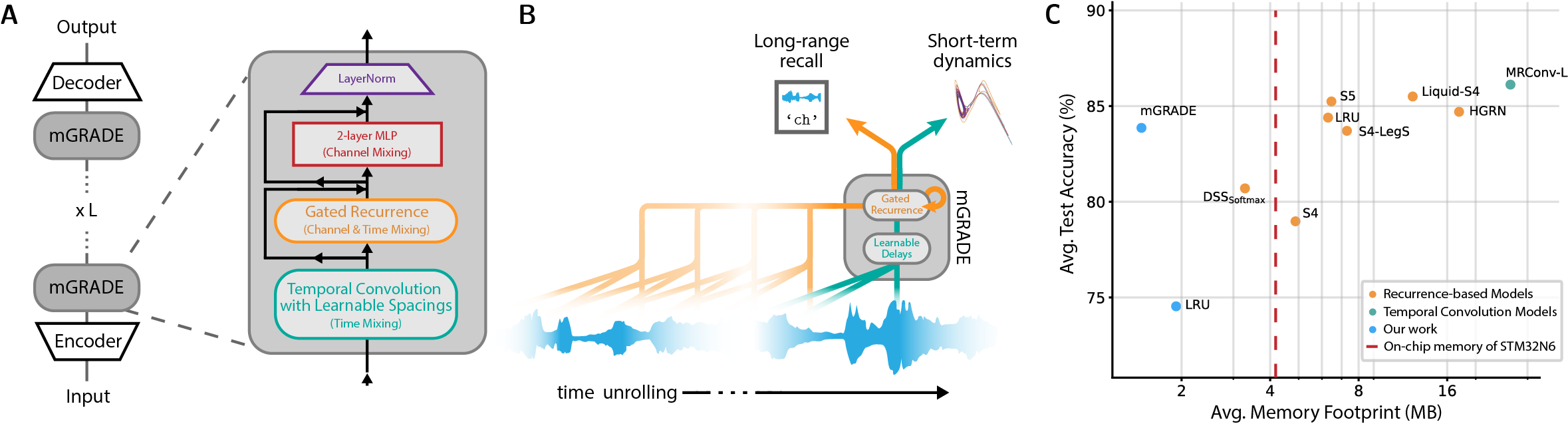}
    \caption{%
    \label{fig:F1}
    \textbf{Novel and theoretically grounded architecture for memory-efficient causal sequence modeling across multiple timescales. A)}
    Network architecture composed of an encoder, $L$ mGRADE layers stacked on each other, and a final decoder. 
    Each layer is composed of four consecutive elements: a modified 1D convolution with learnable spacings, a gated recurrence, an MLP, and layer normalization along with skip connections around the recurrence and MLP.
    \textbf{B)} mGRADE handles multi-timescale information in input sequences (blue waveform) thanks to the complementary functions of its two main components.
    The gated recurrence enables long-range recall by selectively storing long-term dependencies (orange arrows) while the convolution with learnable temporal spacings, equivalent to a delay embedding, models short-term dynamics (green arrows).
    \textbf{C)} mGRADE's average test accuracy across the multi-timescale \gls{lra} tasks is highly competitive while maintaining a memory footprint that is up to $8\times$ smaller than previously published SotA models.
    Average memory footprints are calculated using \cref{tab:lraresults}.
    The blue LRU represents the results of our own experiments when matching LRU to mGRADE in terms of memory size.
    mGRADE is the smallest model with the best average performance which fits within the typical on-chip memory of edge AI systems such as the STM32N6 \citep[][see Appendix \cref{sec:appendix_mcu}]{stm32n6_2025}.
    }
\end{figure*}

Causally modeling sequence data in real-time is a central challenge in machine learning, particularly when dealing with dependencies across multiple timescales.
The difficulty lies in a basic trade-off: capturing long-range dependencies requires storing global information over long time horizons, whereas fast dynamics call for high resolution in the short term.
This trade-off becomes even more prominent when working within a limited memory footprint, such as on embedded systems, impeding temporal processing at the edge for real-time applications such as sensor fusion and autonomous control.
Current \gls{sota} models effectively need to choose between high performance on multi-timescale data or a low enough memory to fit onto embedded systems. 
For example, while \glspl{tcn} \citep{waibel1989phoneme} and Transformers \citep{vaswani2017attention} can handle multi-timescale dependencies, their memory footprints grow with sequence length, preventing real-time inference on streaming data within a fixed memory size.
On the other hand, \glspl{rnn}, particularly gated variants like \gls{lstm} \citep{hochreiter1997long-lstm, gers2000learning-lstm} and \gls{gru} \citep{cho2014learning-gru, chung2014empirical-gru}, offer constant inference-time memory over inputs of arbitrary length, but are inefficient to train and struggle to learn very long-range dependencies \citep{pascanu2013difficulty}.
Even modern, efficiently trainable recurrent models such as \glspl{ssm} \citep{gu2022efficiently} often lack selectivity over long time horizons and exhibit a low-frequency bias that limits the modeling of fast dynamics \citep{gu2024mamba, yu2025tuning}.
Their \gls{sota} performance thus comes at the cost of increased model dimensionality and high parameter precision \citep{zhao2025quantizingsmallscalestatespacemodels}, making them ill-suited for modeling multi-timescale data within the size and precision constraints of edge devices.

Therefore, despite the diversity of existing sequence models, none entirely satisfies the tight memory requirements of embedded systems while effectively capturing both long- and short-range dependencies. 

We address this challenge through three central contributions: (1) a novel sequence model architecture that integrates a recurrent component and a convolution with learnable spacings, (2) a theoretical analysis on why this integration excels at multi-timescale modeling within a smaller memory footprint, and (3) an empirical demonstration of our model's competitive performance and superior memory efficiency.

\textbf{Novel Architectural Integration} 
To meet the memory constraints of embedded systems while retaining the ability to simultaneously model short- and long-range dependencies, we introduce \gls{mgrade}, a compact hybrid-memory sequence model combining a parallelizable gated recurrent component with a causal convolution featuring learnable temporal spacings (Section \cref{sec:modelspec}). 
The learnable temporal spacings go beyond standard 1D-convolutions, enabling \gls{mgrade} to express a parameter-efficient and tuneable delay embedding over the input within a single layer. 
This novel architecture maintains a constant-time memory footprint during inference, supports parallel training regardless of sequence length, and remains entirely causal, making it well-suited for real-time temporal processing in resource-constrained embedded environments.

\textbf{Theoretical Foundations of Hybrid Memory} 
We theoretically analyze how \gls{mgrade}'s two components work together as complementary memory functions to handle different timescales in sequence data efficiently.
Through formal proofs and targeted synthetic tasks, we show that the learnable temporal spacings in the convolution express a parameter-efficient delay embedding, enabling generalization on fast, high-dimensional dynamics.
Additionally, the gated recurrence captures long-range context with high memory efficiency by \textit{selectively} compressing arbitrarily long histories into a fixed-size state.
These findings provide a principled foundation for \gls{mgrade}'s hybrid-memory architecture, where the learnable spacings capture short-range patterns with high parameter-efficiency, and the gated recurrence maintains long-range dependencies without needing to scale the memory size.

\textbf{High Performance with High Memory Efficiency} Finally, we demonstrate that these theoretical capabilities translate to competitive sequence modeling capabilities at a far smaller memory cost than previously shown.
Specifically, we benchmark \gls{mgrade} on the \gls{lra} tasks \citep{tay2021lra} and the 35-way \gls{gsc} raw audio classification task \citep{warden2018speech}, where we achieve competitive results with a memory footprint that is up to 8$\times$ smaller than previously published \gls{sota} models.
Importantly, across all evaluated tasks, \gls{mgrade} is the only model that simultaneously achieves high performance while fitting within the typical memory budgets of edge platforms like the STM32N6 \citep[][more examples in Appendix \cref{sec:appendix_mcu}]{stm32n6_2025}.

By successfully modeling short- and long-range dependencies while satisfying tight memory constraints, \gls{mgrade} directly fills the critical gap in existing sequence models and paves the way for advanced temporal signal processing in resource-constrained embedded systems.

\section{Model Architecture} \label{sec:modelspec}

The \gls{mgrade} architecture consists of an encoder (linear projection), a stack of $L$ \gls{mgrade} layers, and finally a decoder (non-linear projection) (\cref{fig:F1}A). 
The input sequence is streamed element by element, causally producing an output at every timestep $t$. 
The \gls{mgrade} layers are the core architectural feature, combining a modified depthwise 1D convolution with learnable spacings between kernel elements and a parallelizable gated recurrence, followed by a \gls{mlp} and layer normalization. 

\textbf{Convolution with Learnable spacings}
To capture short-term dynamics and high-frequency patterns (Section \cref{sec:attractor}), we first pass the input to each \gls{mgrade} layer through a modified temporal convolution.
To maximize expressivity without exploding the number of parameters, we learn the spacings in time between each kernel element using the \gls{dcls} framework \citep{hassani2023dilated-dcls}, equivalent to learning transmission delays over the input \citep{hammouamri2024learning}.
This choice is inspired by how tunable delays enrich the computational expressivity of spiking neural networks while simultaneously making them more parameter efficient.
\citep{maas1999temporalcoding, dagostino2024denram, goeltz2025delgrad}.

Similar to classical \glspl{tcn}, \gls{dcls} applies a discrete 1D convolution $x_{d,t} = (u_{d} * k_d)[t]$ over every channel $d\leq D$ of the input $\mathbf{u} \in \mathbb{R}^{D\times T}$ (with $T$ being input sequence length).
Unlike classical \glspl{tcn}, the \gls{dcls} convolution kernel $k_d$ is parameterized by \textit{two} sets $\Omega_d=\{w_0, w_1, ..., w_{K-1}\ | \text{ } w_i \in \mathbb{R}\}$ and $\Psi_d=\{p_0, p_1, ..., p_{K-1}\ | \text{ } p_i \in \mathbb{R},\ p_i\leq p_{max}\}$ each of size $K$, representing the weights and positions in time of the kernel elements.
Each position in time $p_i$ is relative to the current timestep, such that $p_i$ represents a single transmission delay.
The maximum position, $p_{max} = \Gamma$, defines the longest possible transmission delay applied to the input, thus indicating the total number of discrete timesteps that $k_d$ spans.
Following \citet{hassani2023dilated-dcls}, we will refer to $\Gamma$ as the \textit{kernel length} and $K$ as the \textit{kernel count}.
Both kernel length and count are the same across all channels.

To construct the discrete kernel $k_d$, each real-valued position $p_i$ is mapped to the discrete kernel indices $n \leq \Gamma$ via a differentiable interpolation function, $c$ (see Appendix \cref{sec:appendix_dcls_kernel}).
This enables both the position and weight of the kernel elements to be learned with gradient descent.
Each element of the kernel $k_d \in \mathbb{R}^\Gamma$ for a single channel $d$ then becomes:
\begin{align}
\ k_d[n] &= \sum_{i=0}^{K-1} w_i \cdot c[n, p_i], \quad
\label{eq:kernel_dcls}
\end{align}

The \gls{dcls} convolution's output $x_{d,t}$ for each channel $d$ at timestep $t$ is then computed as a 1D convolution using the $k_d$ kernel over the past $\Gamma$ timesteps of the input.
We stack all $D$ kernels $k_d$ into a kernel matrix $\mathbf{K} \in \mathbb{R}^{D \times \Gamma}$, yielding a final output vector $\mathbf{x}_t \in \mathbb{R}^{D}$ after the convolution.

\textbf{Gated recurrent component}
To enable \gls{mgrade} to selectively model long-range dependencies (Section \cref{sec:flipflop}), we include a gated recurrent component in the \gls{mgrade} layer.
To this end, we use a simplified and parallelizable version of the \gls{gru} \citep{cho2014learning-gru}, initially proposed by \citet{martin2018parallelizing}, and also known as \gls{mingru} \citep{feng2025were}.
\gls{mingru} removes the dependency of the update gate, $\mathbf{z}_t\in \mathbb{R}^{H}$, and candidate activation, $\mathbf{\tilde{h}}_t\in \mathbb{R}^{H}$, on the previous hidden state, $\mathbf{h}_{t-1}\in \mathbb{R}^{H}$.
In our case, the hidden dimensionality $H$ is equal to $D$ times some expansion factor.
Given the output of the \gls{dcls} $\mathbf{x}_t\in \mathbb{R}^{D}$, $\mathbf{h}_t\in \mathbb{R}^{H}$ is updated as follows,
\begin{equation}
    \begin{aligned}
    &\mathbf{h}_t = (1 - \mathbf{z}_t) \odot \mathbf{h}_{t-1} + \mathbf{z}_t \odot \mathbf{\tilde{h}}_t \\ & \mathbf{z}_t = \sigma(\mathbf{W}_z \mathbf{x}_t)\, ,\quad  \mathbf{\tilde{h}}_t = \mathbf{W}_h \mathbf{x}_t 
    \end{aligned}
    \label{eq:mgrade_oneliner}
\end{equation}
where $\sigma$ is the sigmoid function, $\odot$ is the Hadamard product, and $\mathbf{W}_z$ and $\mathbf{W}_h \in \mathbb{R}^{D\times H}$ are the weights of the projections for $\mathbf{z}_t$ and $\mathbf{\tilde{h}}_t$, respectively.

The specific choice of a \gls{mingru}-style gated recurrence is motivated by its training efficiency and hardware compatibility.
Since the update gate and candidate activation only depend on the current $\mathbf{x}_t$, hidden states for every timestep can be computed in parallel using a prefix scan \citep{blelloch1990prefix}, enabling efficient training in logarithmic time with respect to sequence length \citep{feng2025were}.
In addition, this architecture is well-suited to heavily quantized, low-power hardware implementations \citep{billaudelle2025minimalist}.

\textbf{MLP and Layer Normalization}
Following the gated recurrent component, the hidden state $\mathbf{h}_t$ is passed through an \gls{mlp} with $\mathbf{W}_{\text{MLP,in}}\in \mathbb{R}^{H\times 2D}$, $\mathbf{W}_{\text{MLP,out}}\in \mathbb{R}^{2D\times D}$.
Afterwards, layer normalization is applied.
Since $D$ is the dimensionality of the activations passed between layers, we call it the model dimensionality.

\textbf{Memory Complexity}\label{sec:memory_complexity}
During inference, \gls{mgrade} requires memory for both the model parameters and for the activation buffer, which stores all past and current activations needed to produce an output for the current timestep.
The number of model parameters scales primarily with the model dimensionality $D$.
Regarding the activation buffer, the gated recurrent component utilizes only a fixed-size hidden state vector of size $H$.
It can therefore operate over arbitrary sequence lengths without scaling the activation buffer size, allowing us to fix the kernel length $\Gamma$ (or maximum delay) of the convolutional component while maintaining a theoretically unbounded temporal \textit{receptive field}\footnote{range of past inputs that can influence the output at any $t$}.
Thus, \gls{mgrade}'s memory complexity is independent of the input sequence length, in marked contrast to architectures like Transformers and \glspl{tcn}, where memory requirements scale at least linearly with the sequence length or temporal receptive field. Further details can be found in Section \cref{sec:appendix_memoryfootprint}.

\section{Theoretical Foundations of Hybrid Memory} \label{sec:theoretical}
  We now develop a theoretical understanding of how \gls{mgrade}'s learnable temporal spacings in the convolution and gated recurrence complement each other.
To this end, we first investigate how the learnable spacings enhance \gls{mgrade} beyond purely recurrent architectures by strengthening its structural inductive bias towards the reconstruction of short-term dynamics.
We then show how the gated recurrent component enables long-range dependency learning by showing that a single \gls{mgrade} layer can formally model the Flip-Flop language and empirically solve the selective copying task, both of which require selectively remembering long-range dependencies that cannot be modeled by purely convolutional (e.g. \glspl{tcn}) or non-gated recurrent architectures (e.g. time-invariant \glspl{ssm}).

\subsection{Learnable Spacings capture Short-Term Dynamics}\label{sec:attractor}

\begin{figure}[t]
    \centering \includegraphics[width=1\columnwidth]{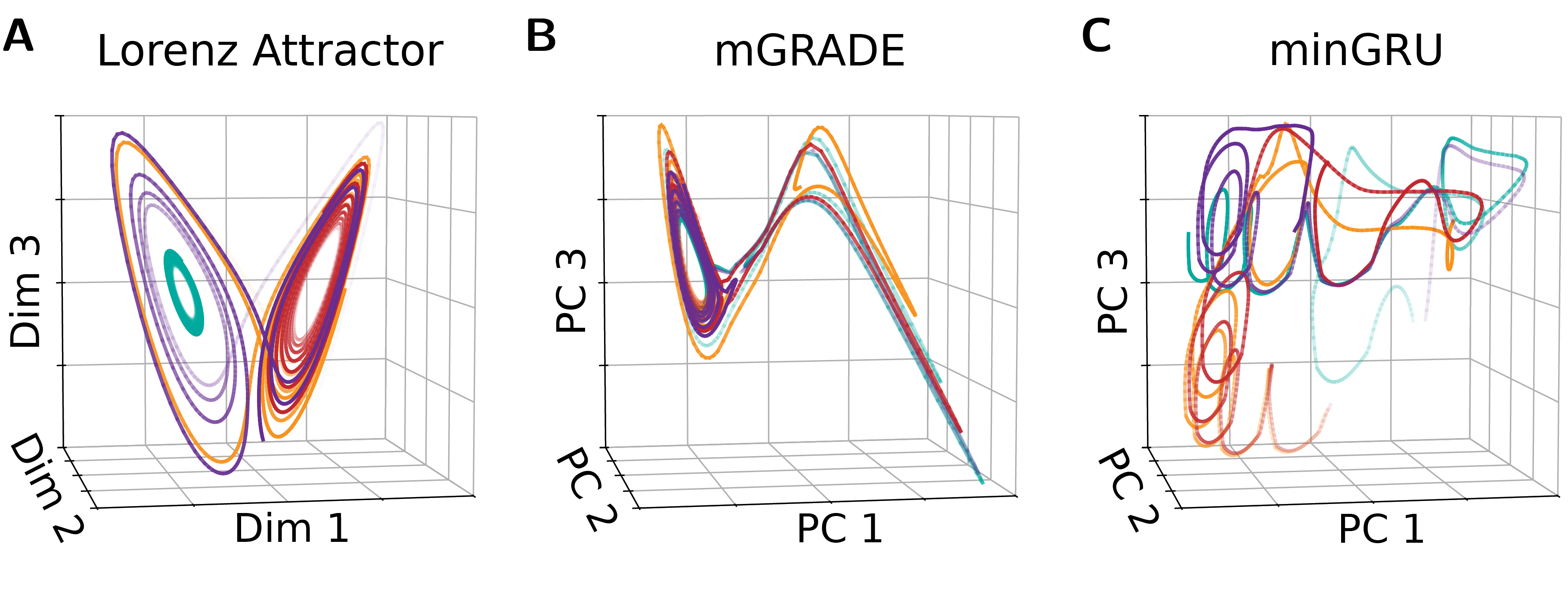}    
    \caption{
    \textbf{\gls{mgrade} reconstructs a diffeomorphic mapping of the input dynamics.}
    \textbf{A)} Representative trajectories ($n=4$) on the Lorenz attractor manifold with 5\% Gaussian time-independent noise. The task is to predict dimension 1 at the next timestep. 
    \textbf{B)} Representative trajectories in the hidden state space of a single-layer \gls{mgrade} projected to the first 3 Principal Components (PC).
    \textbf{C)} Trajectories in the hidden state space of a 2-layer \gls{mingru} projected to the first 3 Principal Components (PC). See \cref{fig:SI_lorentz_2d_shapes} for all PCs compared individually.
    }
    \label{fig:F2}
\end{figure}

\gls{mgrade}'s modified temporal convolution can be reframed as computing weighted sums of time-delayed inputs stored in an activation buffer, with the learnable spacings controlling the delay durations.
This mirrors delay embeddings, a dynamical systems technique used for time-series prediction and state-space reconstruction \citep{strogatz2015dynamics}. 

Delay embeddings map an input sequence to a higher-dimensional vector consisting of $m$ time-delayed copies of the original input. 
Takens’ Embedding Theorem \citep{Takens1981} guarantees that, for a $d$-dimensional dynamical system, any delay embedding of even a single observed dimension can diffeomorphically reconstruct the underlying manifold along which the system moves, using at most $m=2d+1$ delays in noise-free conditions. 
Intuitively, this means that a vector made up of $m$ different delayed input copies traces out trajectories in $m$-dimensional space that resemble the underlying $d$-dimensional dynamical system's trajectories up to a smooth, invertible transformation. 

\begin{theorem}[Informal]\label{theorem:mgrade_delay_embedding} A single-layer \gls{mgrade} can express a delay embedding of a $d$-dimensional dynamical system in the sense of \citet{Takens1981}, using only an $m$-dimensional projection of a single observed dimension as input. 
Its $m$-dimensional hidden state can thus learn to diffeomorphically reconstruct the full geometry of the system's trajectories.
\end{theorem}

The full theorem and proof are in Appendix \cref{sec:appendix_lorenz_proof}. It relies on the fact that \gls{mgrade} can learn distinct delays for each of the $m$ projections of the observed dimension, and then embed them directly into its hidden state.
The resulting internal representation captures the full geometry of the underlying dynamical system, allowing \gls{mgrade} to generalize to dimensions that were unobserved during training.

We validate this theoretical capability on a next-step prediction task using the chaotic 3D-Lorenz attractor, training a single-layer \gls{mgrade} and a 2-layer \gls{mingru}\footnote{For a fair comparison, we give the \gls{mingru} an additional layer to provide its update gate and candidate activation in the second layer with temporal information. This means that the only significant difference between the \gls{mgrade} and \gls{mingru} models is the fact that \gls{mgrade}'s uses a convolution.} on 2000 noisy trajectories \citep[\cref{fig:F2}A;][]{Lorenz1963}.
To quantify the quality of the predictions, we use the \gls{mase} \citep{hyndmanMASE2006}.
Note that \gls{mase} $>$ 1 indicates that a model has no predictive power relative to naively predicting the current state's persistence (Appendix \cref{sec:appendix_mase}).

\begin{table}[t]
\caption{\textbf{Next-step prediction on 3D-Lorenz attractor.}}
\label{tab:lorenzresults}
\vspace{-0.2cm}
\begin{center}
\renewcommand{\arraystretch}{1.2}
\resizebox{\linewidth}{!}{
\begin{tabular}{lcccc}
\toprule
Model & MASE & MASE  & Near.\ Neigh.\ & Params. / Buff. \\
 & observed dim. & unobserved dim. &  Overl.\ \% &  (bytes)\\
\midrule
 \textbf{\gls{mgrade}} & \textbf{0.38 \textpm~0.02} & \textbf{0.86 \textpm~0.11} & \textbf{32.7 \textpm~0.7} & \textbf{1.1K} / 128\\
minGRU & 0.63 \textpm~0.01 & 1.01 \textpm~0.01 & 28.8 \textpm~2.2 & 1.8K / \textbf{80}\\
\bottomrule
\end{tabular}
}
\end{center}
\end{table}

Visualizing the top three Principal Components of the 10 hidden states (\cref{fig:F2}B,C; \cref{fig:SI_lorentz_2d_shapes}), \gls{mgrade}'s embedding reconstructs the Lorenz system's characteristic two-lobe structure, while the \gls{mingru}'s embedding lacks similar visual correspondence.
\gls{mgrade} also achieves a \gls{mase} that is 1.6$\times$ lower than \gls{mingru} when predicting the next timestep on observed dimensions (\cref{tab:lorenzresults}; \cref{fig:SI_lorentz_over_epoch}). 
When predicting the dimensions unobserved during training, \gls{mgrade} outperforms the 2-layer \gls{mingru}, which shows no predictive power (\gls{mase} $>$ 1).
We also quantify how smoothly the geometry of the original attractor maps to the geometry of the hidden state space by measuring Nearest neighbor Overlap following \citep{Ostrow2024} (Appendix \cref{sec:appendix_lorenz_details}).
A high overlap percentage indicates that locally the two manifolds are smooth invertible mappings of each other, i.e., that the hidden space is a faithful diffeomorphic reconstruction of the original.
Consistent with our visual check, \gls{mgrade} exceeds the \gls{mingru} by 4\%.

This experiment highlights how the temporal convolution with learnable spacings, essentially a short-term buffer of delayed inputs, implements a compact representation for reconstructing dynamical state spaces from partial observations over a short-term horizon.
Importantly, the number of required delays (and thus parameters) grows linearly with the dimensionality of the underlying dynamical system, explaining \gls{mgrade}’s parameter efficiency compared to models whose memory scales with sequence length \citep[][Appendix \cref{sec:appendix_memoryfootprint}]{schlag2021lineartransformerssecretlyfast}.

Appendix \cref{sec:appendix_highfreqtask} expands on these insights by showing that the convolution with learnable spacings enables \gls{mgrade} to recognize and respond to high-frequency features far better than purely recurrent architectures, overcoming the bias towards low-frequency information shared by gated \glspl{rnn} and \glspl{ssm} \citep{rahaman2019spectralbias, yu2025tuning}.

\subsection{Gated Recurrence stores Long-Range Dependencies}\label{sec:flipflop}

\begin{figure}[t]
    \centering \includegraphics[width=0.8\columnwidth]{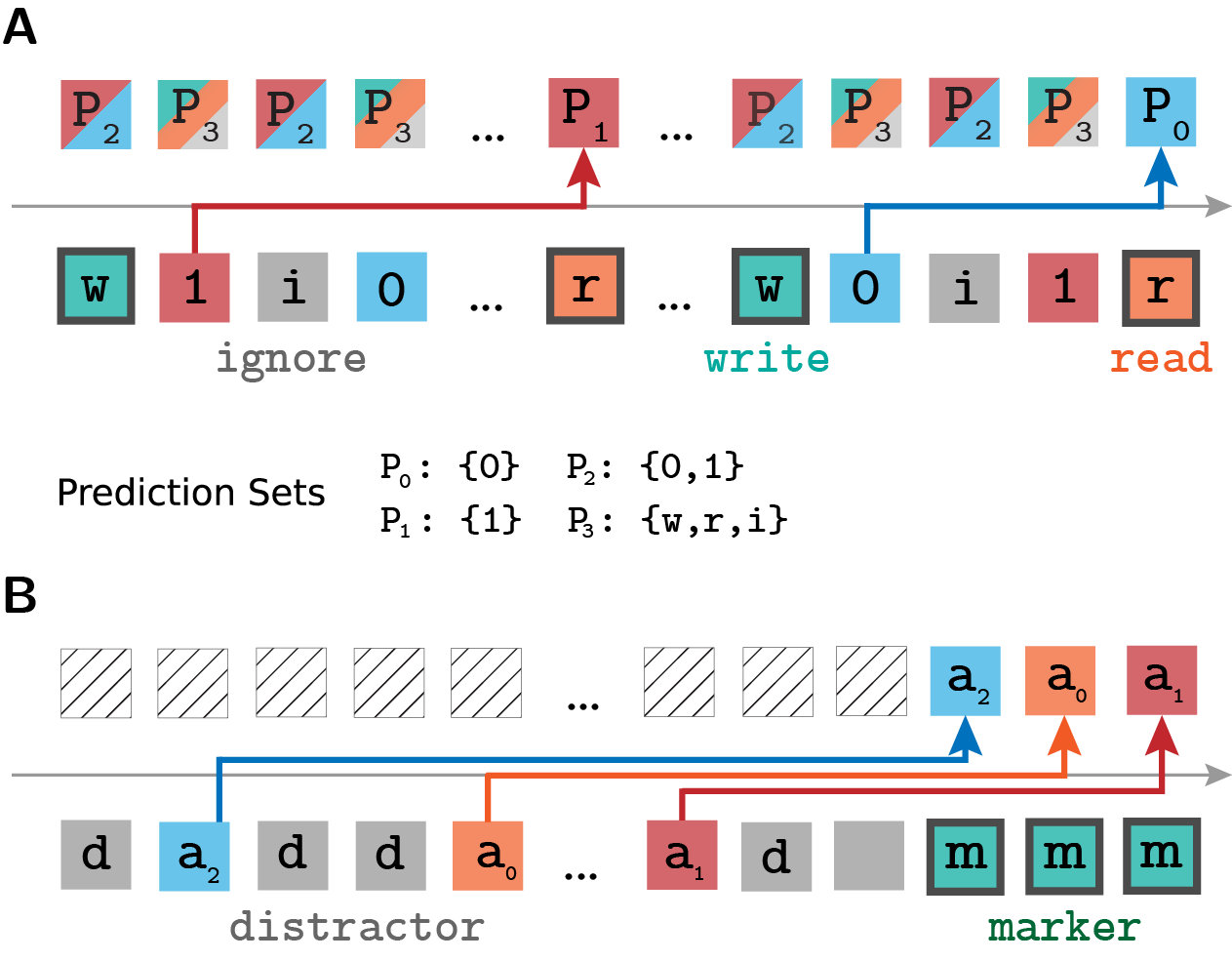}    
    \caption{\textbf{Illustration of Flip-Flop and selective copying tasks.}
    \textbf{A)} Flip-Flop modeling consists of predicting the \textit{prediction set} $P_i$ of next possible symbols in the given Flip-Flop string at every timestep.
    For \texttt{r} symbols, this is equivalent to recalling the value after the last \texttt{w}. 
    \textbf{B)} Selective copying requires recalling (after the marker symbol \texttt{m}) randomly distributed value symbols \texttt{a}$_i$ in the order they are presented while ignoring distractor symbols \texttt{d}.}
    \label{fig:F3}
\end{figure}

To probe \gls{mgrade}'s ability to \textit{selectively} remember long-range dependencies, we analyze its ability to predictively model Flip-Flop languages, a formal language family designed to test sequence models’ long-range capabilities. \citep[\cref{fig:F3}A;][]{liu2023attentionglitch}.
\begin{defn}[Flip-Flop Language]
    Let the alphabet be \( \Sigma = \{\texttt{w, r, i, 0, 1}\} \), where \(\texttt{w}\), \(\texttt{r}\), and \(\texttt{i}\) represent the instruction symbols for ``write", ``read", ``ignore", and \(\texttt{0}\), \(\texttt{1}\) represent value symbols. 
    Flip-Flop languages \( L_{ff} \) consist of sets of strings over \( \Sigma \) that being with \(\texttt{w}\) and then alternate between instructions and values (e.g., \(\texttt{w 0 r 0 i 1}\)), satisfying the condition that after any \(\texttt{r}\), the subsequent symbol is equal to the value following the previous \(\texttt{w}\)..
\end{defn}
\begin{defn}[Predictive Modeling]
    For a string \( s \in L_{ff} \) and a prefix \( s[1:t] \) ending at position \( t \) with symbol \( a_t \), predictive modeling requires outputting the \textit{prediction set} \( P_i \subseteq \Sigma \) of valid next symbols \( a_{t+1} \) such that \( s[1:t] \, a_{t+1} \) remains a prefix of a string in \( L_{ff} \). 
    We say that a model \textit{predictively models} \( L_{ff} \) iff its output at each timestep \(t\) encodes all the information needed for a linear classifier to return the next prediction set with no errors.
\end{defn}

In the Flip-Flop task, the model must predict at each timestep which class of next symbols can follow the current one such that the sequence remains a valid Flip-Flop string.
For instance, when the model sees a \texttt{w} it should return the class including \texttt{0} \textit{and} \texttt{1} (because both are valid next symbols), while after an \texttt{r} it depends on the value after the most recent \texttt{w}, which means either \texttt{0} \textit{or} \texttt{1}.

Predictive modeling of Flip-Flop languages is interesting for multiple reasons. First, a model's success on Flip-Flop modeling implies a broad computational expressivity on multiple formal languages and algorithmic simulation tasks \citep{liu2023attentionglitch}. 
Second, Flip-Flop modeling requires maintaining the last \texttt{w}-paired value over arbitrarily long sequences.
Accordingly, models with fixed-length context windows or sequence-length dependent memory scaling, such as \glspl{tcn} and Transformers, cannot model Flip-Flop over arbitrary lengths with a fixed memory size \citep[][proof in Appendix \cref{sec:appendix_flipflop_tcn_proof}]{sarrof2024, liu2023attentionglitch}.
Finally, since any \texttt{r} is typically separated from the most recent \texttt{w} by an arbitrarily long string of irrelevant \texttt{i}-paired values, Flip-Flop modeling requires selectively remembering and ignoring value symbols based on the content of the preceding instruction.
This content-aware \textit{selectivity} across time factors into many real-world challenges, such as tracking filler-gap dependencies in natural language \citep{wilcox2018fillergap, howitt2024} or ignoring irrelevant inputs during arithmetic reasoning \citep{shi2023}.
Notably, time-invariant \glspl{ssm} without input-dependent gating, such as the \gls{h3} \citep{dao2022hungry-h3} or \gls{lru} \citep{orvieto2023resurrecting-lru}, lack selectivity \citep{gu2024mamba}.
In contrast to \glspl{tcn}, Transformers, and linear time-invariant \glspl{ssm}, a single-layer \gls{mgrade} can predictively model Flip-Flop languages due to its gated recurrent component.
\begin{theorem}[Flip-Flop Modeling with mGRADE]\label{theorem:mgrade} A single-layer \gls{mgrade} with at least 2 delays can predictively model a Flip-Flop language, $L_{ff}$, at arbitrary length. %
\end{theorem}

\textit{Proof Sketch.} (Full proof in Appendix \cref{sec:appendix_flipflop_proof}) \gls{mgrade} stores the value after the last $\texttt{w}$ in one part of its hidden state, while the other merely reproduces the input. 
Learnable delays trigger the update gate of the storage hidden state only after a $\texttt{w}$ (\textit{selectively}), which then preserves the value over arbitrary sequence lengths via its recurrence until the next $\texttt{w}$.
A linear classifier can then trivially extract the prediction set by reading the current input symbol from the reproducing hidden state, and, if the current input is $\texttt{r}$, reading out the stored value from the storage hidden state.

To validate this proof, we train a single-layer \gls{mgrade}, a single-layer linear time-invariant \gls{ssm} (\gls{lru} augmented with the \gls{dcls} convolution), and a 5-layer \gls{tcn} on the Flip-Flop dataset by \citet{liu2023attentionglitch}.
The training data consists of 1.6M Flip-Flop strings of 512 timesteps, where the expected distance between $\texttt{w}$ and $\texttt{r}$ is 10 timesteps. 
For testing, we used out-of-distribution data of the same length with sparse $\texttt{w}$ and $\texttt{r}$ (expected distance around 100 timesteps) to stress long-range dependency learning. 
We report the recall accuracy, i.e., how often the model predicted the value following any given $\texttt{r}$ symbol correctly.
\gls{mgrade} solves the task to nearly $100$\%, significantly outperforming both the \gls{tcn} and the \gls{lru} despite using less parameters or a smaller activation buffer, even at longer distances between successive $\texttt{w}$ and $\texttt{r}$ symbols (\cref{fig:SI_flipflop}).

In addition to Flip-Flop language modeling, we further evaluate \gls{mgrade}'s selectivity relative to linear time-invariant \glspl{ssm} and \glspl{tcn} by comparing the performance of \gls{mgrade} to \gls{lru} and a \gls{tcn} on the well-established selective copying task \citep[\cref{fig:F3}B;][]{tegmark2019orthogonalru, gu2024mamba}. 
Selective copying is related to Flip-Flop modeling, but instead of providing an explicit instruction ahead of the sequence elements that contain relevant content (i.e.\ the $\texttt{w}$ symbol used in $L_{ff}$), the content of the sequence element itself defines its relevance (see Appendix \cref{sec:appendix_selectivecopy} for details).

\begin{table}[t]
\caption{\textbf{Test accuracy on Flip-Flop and Selective Copying.} Results on H3 \cite{dao2022hungry-h3} are from \citet{gu2024mamba}. Parameter (``Params.".) and activation memory (``Buff.") in bytes.}
\label{tab:flipflopselective}
\centering
\resizebox{\linewidth}{!}{
\setlength{\tabcolsep}{3pt}
\begin{threeparttable}
\begin{tabular}{lcccc}
    \toprule
                            & \multicolumn{2}{c}{Flip-Flop} & \multicolumn{2}{c}{Selective Copying}\\
                            \cmidrule(lr){2-3}\cmidrule(lr){4-5}
                            & Acc.  & Params. / Buff. &  Acc.  & Params. / Buff. \\
    \midrule
    \textbf{\gls{mgrade}}   & \textbf{99.6 \textpm~0.3} & \textbf{12K} / \textbf{384} & \textbf{87.1 \textpm~2.2} & \textbf{260K} / \textbf{68K} \\
    \gls{lru} + \gls{dcls}  & 88.5 \textpm~3.5          & 20\text{K} / 384          & 16.7 \textpm~3.5          & 324\text{K} / 68\text{K} \\
    \gls{tcn}               & 60.6 \textpm~0.1          & 12K / 64K  & 20.4 \textpm~0.1                     & 320K / 3.2M \\
    H3                      & --                     & --          & 57.0                  & 664K / 2K \\
    \bottomrule
\end{tabular}
\end{threeparttable}
}
\end{table}

Our results in \cref{tab:flipflopselective} show that \gls{mgrade} clearly outperforms our \gls{lru} implementation by more than 70\%, despite using 1.3$\times$ less parameters.
It also outperforms the \gls{tcn} by more than 60\% despite the \gls{tcn}'s receptive field being constructed to cover the entire input sequence (see \cref{tab:selectivecopy-hp}).
These results are consistent with the results reported by \citet{gu2024mamba} on \gls{h3} \citep{dao2022hungry-h3}, another linear time-invariant \gls{ssm}.
This emphasizes the importance of the input-dependent gating over hidden states, which linear time-invariant \glspl{ssm} and \glspl{tcn} lack.

Overall, the Flip-Flop modeling and selective copying tasks confirm that \gls{mgrade}’s minGRU-style recurrent component enables robust and selective long-range dependency modeling, outperforming purely convolutional and non-gated recurrent models without having to scale the memory size with input sequence length. These long-range learning capabilities combined with the ability of the learnable spacings to model short-term dynamics underpin \gls{mgrade}’s strong performance on real-world sequence tasks.

\section{Empirical Evaluation} \label{sec:empiricalverification}
  We empirically evaluate \gls{mgrade}'s capabilities on two complementary sequence modeling benchmarks, \gls{lra} and \gls{gsc}, designed to contain challenging long- and short-range dependencies.
The \gls{lra} benchmark \citep{tay2021lra} assesses the performance and inductive biases of \gls{mgrade} on long range dependency tasks across different modalities (text and flattened images), featuring sequences of lengths 1K to 8K. %
The \gls{gsc} task evaluates \gls{mgrade} on real-world time-series data, requiring the classification of raw speech recordings into one of 35 classes.
Each audio sample is a one-second waveform recorded at 16 kHz, yielding sequences of length 16K.
When evaluating the results, we do not solely focus on the achieved accuracy, but also consider the memory footprints of both parameters ("Params.") and activation buffer ("Buff."), which indicate how suitable the model is for deployment on embedded systems.
Appendix \cref{sec:appendix_activ_buff_calc} and \cref{sec:appendix_embedded_memory} show the calculation of activation buffer sizes and considerations regarding the instantiation on embedded platforms, respectively.

\textbf{Experimental setup}
Since \gls{mgrade} is designed for real-time signal processing at the edge, it processes its inputs in a streamed and causal fashion.
Therefore, we deliberately avoid using the (acausal) bidirectional processing used by S5 \citep{smith2023simplified}, S4-LegS \citep{gu2022on}, and HGRN \citep{qin2023HGRN}, although its inclusion leads to significant performance improvements on \gls{lra}, as shown by comparing S4-LegS to the causal S4 \citep{gu2022efficiently} on Retrieval, Image, and Pathfinder (row 1 vs row 2 in \cref{tab:lraresults}). 
This also holds for raw-speech classification, with bidirectional models achieving a 3-5\% gain (top three rows of \cref{tab:sc35_compact}).
Additionally, we process the raw inputs directly, consistent with all baselines, except for HGRN \citep{qin2023HGRN} which applies positional encoding to all \gls{lra} tasks except Pathfinder.
Experiment details are in Appendix \cref{sec:appendix_hyperparams}.

\textbf{Results}
In \cref{tab:lraresults}, we compare \gls{mgrade}'s performance on \gls{lra} to current \gls{sota} \gls{rnn} and convolution-based architectures.
Compared to the best-performing models, \gls{mgrade} reduces the memory footprint significantly, while still achieving comparable accuracies:
for example,  on ListOps, it achieves an accuracy within 0.9\% of Liquid-S4's performance \citep{hasani2023liquid}, while using 7$\times$ fewer parameters; on Pathfinder, it remains within 1.8\% of MRConv-L's performance, while using 8$\times$ smaller activation buffers.
Compared to the models that are closest in size, \gls{mgrade} delivers higher performance:
1.9\% higher accuracy than HGRN \citep{qin2023HGRN} on ListOps, 8\% higher accuracy than S4 on Pathfinder (while still using 1.5-2$\times$ fewer parameters).
With a maximum memory usage below 3MB for the Image task, \gls{mgrade} is the only model compatible with the 4MB on-chip memory of edge devices like the STM32N6 (see Appendix \ref{sec:appendix_mcu}); all other models exceed this limit, requiring over 7 MB.
We note 

While the previous comparisons demonstrate that \gls{mgrade} attains comparable accuracy with substantially fewer parameters, an important complementary question is how the best performing models behave under iso-parameter conditions. 
To assess this, we optimize \gls{lru} to more closely match \gls{mgrade} in both parameter count and activation footprint. Under this iso-capacity setting, \gls{mgrade} achieves an average 9\% higher accuracy across all \gls{lra} tasks (except PathX), indicating a clear architectural advantage beyond parameter efficiency alone.
Note that unlike the results reported in \citet{orvieto2023resurrecting-lru}, our \gls{lru} and \gls{mgrade} implementations are fully causal for every task, thus suitable for real-time sequence processing.

We note that \gls{mgrade} achieves strong results across \gls{lra} but only reaches chance level on the synthetic PathX task. 
This likely reflects a mismatch in inductive bias rather than a sequence length limitation, as PathX's sparse, flattened 2D fundamentally differ from the more natural multi-timescale dependencies \gls{mgrade} is designed for. 
To test this, we evaluate \gls{mgrade} on the 16K \gls{gsc} task which matches PathX in length but better reflects real-world signals.

In \cref{tab:sc35_compact}, we present the raw-speech classification results on \gls{gsc}, comparing to current \gls{sota} recurrent architectures. 
\gls{mgrade} attains an accuracy within 2.1\% of Liquid-S4 \citep{hasani2023liquid} while requiring 10\% fewer parameters. 
In addition, \gls{mgrade} relies solely on real-valued operations, avoiding the complex-valued arithmetic used in Liquid-S4. 
This combination of reduced parameter count and simpler computation makes \gls{mgrade} a suitable candidate for deployment under hardware or energy constraints.

These results confirm that our proposed architecture is capable of tackling large-scale and long-range tasks, particularly for time-series data, thus validating our theoretical predictions and demonstrating clear advantages in both memory footprint and performance.
Appendix \cref{sec:appendix_spiking} also assesses whether these advantages persist relative to efficient architectures tailored for edge deployment, such as \glspl{snn}. 
\gls{mgrade} remains competitive in terms of parameter numbers, achieving accuracy within 2.6\% of the strongest \gls{snn} while still requiring 3$\times$ less parameters.

To evaluate the impact of \gls{mgrade}'s architectural components on performance, we perform an ablation study on three \gls{lra} tasks, with results summarized in \cref{tab:lraablation_compact}. 
While the convolutions with learnable spacings (DCLS) and recurrent components (\gls{mingru}) perform well on their own on Image and ListOps respectively, only full \gls{mgrade} successfully tackles \textit{both} tasks.
Additionally, both components are needed to solve Pathfinder above chance level.
A more detailed analysis is provided in Appendix \cref{sec:appendix_ablation}

\begin{table*}[t]
\caption{\textbf{Test accuracy on the \gls{lra} benchmark.} We separate parameter memory requirements (``Params.", in bytes) from activation buffer memory requirements (``Buff.", in bytes) (see Appendix \cref{sec:appendix_memoryfootprint}). Parameter counts and accuracies not made available in the publications or extractable from the official code and hyperparameters are denoted by a dash. Best results are in bold, second best underlined.}
\label{tab:lraresults}
    \resizebox{\textwidth}{!}{
    \begin{threeparttable}    
    \vspace{0.1cm}
    \centering
    \begin{tabular}{lcccccccccc}
    \toprule
             & \multicolumn{2}{c}{ListOps} & \multicolumn{2}{c}{Text} & \multicolumn{2}{c}{Retrieval} & \multicolumn{2}{c}{Image} & \multicolumn{2}{c}{Pathfinder} \\ 
        \cmidrule(lr){2-3}\cmidrule(lr){4-5}\cmidrule(lr){6-7}\cmidrule(lr){8-9}\cmidrule(lr){10-11}
        Model & Acc. & Params. / Buff. & Acc. & Params. / Buff. & Acc. & Params. / Buff. & Acc. & Params. / Buff. & Acc. & Params. / Buff. \\ 
    \midrule
        \textit{RNN-based architectures} & & & & & & & & & & \\
        S4~\citep{gu2022efficiently}\tnote{5}                     & 58.4 & 996.1K / 191.4K              & 76.0 & 718.8K / 62.5K                          & 87.1 & 4.58M / 382.8K                                    & 87.3 & 12.98M / 769.5K                  & 86.1           & 3.42M / \underline{382.8K}  \\ 
        S4-LegS~\citep{gu2022on}\tnote{1,5}                       & 59.6 & 2.28M / 511.7K               & 86.8 & 4.96M / 769.5K                          & 90.9 & 6.11M / 769.5K                                    & 88.7 & 13.74M / 1.50M                   & 94.2           & 4.96M / 769.5K\\ 
        DSS$_\text{SOFTMAX}$ ~\citep{gupta2022diagonal}\tnote{5}  & 60.6 & 804.7K / 191.4K              & 84.8 & 593.8K / 62.5K                          & 87.8 & 3.39M / 382.8K                                    & 85.7 & 7.63M / 769.5K                   & 84.6           & \textbf{2.29M} / \underline{382.8K} \\ 
        Liquid-S4~\citep{hasani2023liquid}\tnote{5}               & \textbf{62.8} & 1.29M / 31.3K       & 89.0 & 640.6K / 15.6K                          & 91.2 & 5.73M / 382.8K                                    & \underline{89.5} & 41.98M / 6.11M       & 94.8           & 4.58M / \underline{382.8K}   \\ 
        S5~\citep{smith2023simplified}\tnote{1,5}                 & 62.2 & 742.2K / \textbf{0.4K}      & \underline{89.3} & 4.96M / 4.3K                & 91.4 & 2.94M / 5.9K                                     & 88.0 & 19.47M / \textbf{9.0K}          & 95.3           & 4.30M / \textbf{5.9K}  \\
        \gls{lru}~\citep{orvieto2023resurrecting-lru}\tnote{3}    & 60.2 & 742.2K / 5.9K               & \textbf{89.4} & 4.96M / 4.3K      & 89.9 & 2.94M / 5.9K                                    & $-$  & $-$ / $-$                        & 95.1\tnote{1}  &  4.30M / \textbf{5.9K} \\ 
        Mamba~\citep{soydan2024s7}\tnote{7}                     & 38.0 & $-$ / $-$              & 83.0 & $-$ / $-$                          & 73.0 & $-$ / $-$                                    & 70.0 & $-$ / $-$                  & 69.0           & $-$ / $-$  \\ 
        HGRN~\citep{qin2023HGRN}\tnote{1, 2, 5}                   & 60.0 & 328.1K / 1.6K               & 88.1 & 3.35M / \underline{3.9K}                           & \textbf{94.2} & \underline{449.2K} / \textbf{1.2K}      & 88.7 & 78.63M / \underline{23.83K}                  & 92.9           & 4.96M / \textbf{5.9K}  \\ 
    \midrule
        \textit{Convolution-based architectures} & & & & & & & & & &  \\
        SGConv~\citep{li2023what-sgconv}\tnote{4}                 & 61.5 & $-$ / $\sim$ 5.73M              & 89.2 & $-$ / $\sim$ 3.82M                     & 91.1 & $-$ / $\sim$ 24.05M                             & 88.0 & $-$ / $\sim$ 11.83M                       & \underline{95.5} & $-$ / $\sim$ 6.11M \\ 
        MRConv-L ~\citep{cunningham2024reparameterized}\tnote{4}  & \underline{62.4} & 2.52M / $\sim$ 5.73M & \textbf{89.4} & $-$ / $\sim$ 3.82M                     & \underline{91.5} & $-$ / $\sim$ 24.05M                & \textbf{90.6} & 29.39M / $\sim$ 11.83M             & \textbf{96.7} & $-$ / $\sim$ 6.11M \\ 
    \midrule
        \textit{Our implementation} & & & & & & & & & &  \\
        {\gls{lru}}                                               & 58.3 & \underline{164.1K} / \underline{0.78K} & 85.9 & \underline{179.7K} / \textbf{0.8K}                 & 86.8 & 460.9K / \underline{1.6K}                  & 84.3  & \underline{4.58M} / \textbf{9.0K}    & 57.4\tnote{6}  & 4.30M / \textbf{5.9K}  \\ 
        {\gls{mgrade}}                                            & 61.9 & \textbf{156.3K} / 11.7K              & 87.3 & \textbf{171.9K} / 5.9K                 & 88.1 & \textbf{406.3K} / 6.6K                         & 87.1  & \textbf{2.65M} / 769.5K                    & 94.9           & \underline{2.33M} / 769.5K  \\ 
    \bottomrule
    \end{tabular}\begin{tablenotes}
       \item [1] Bi-directional input processing.
       \item [2] Uses positional encoding of the input.
       \item [3] Assuming same hyperparameters as in S5 \citep{smith2023simplified} as mentioned in \citep{orvieto2023resurrecting-lru} (official code not available).
       \item [4] Buffer sizes calculated assuming same hyperparameters as in S4 \citep{gu2022efficiently} as mentioned in \citep{li2023what-sgconv} and \citep{cunningham2024reparameterized} (code not available).
       \item [5] Parameter numbers extracted from the official GitHub repositories.
       \item [6] Best validation accuracy using fully causal model (see Appendix \cref{sec:appendix_hyperparams}).
     \end{tablenotes}
    \end{threeparttable}
    }
\end{table*}

\begin{table}
  \caption{\textbf{Test accuracy on the 35-way \gls{gsc} classification task.} We differentiate causal and bidirectional architectures as in \citep{gu2022on}. Best results are in bold, second best underlined.}
  \label{tab:sc35_compact}
    \vspace{0.1cm}
    \centering
    \resizebox{\columnwidth}{!}{
    \begin{threeparttable}    
        \begin{tabular}{@{}lccc@{}}
        \toprule                   
        Model                                         & Params. / Buff.  & Causal     & Bidirectional    \\ 
        \midrule
        S4-LegS     \citep{gu2022on}\tnote{1,2}                  & 1.17M / 191K          & 93.6               & 96.1 \\
        S4-FouT     \citep{gu2022on}\tnote{1,2}                  & 1.17M / 191K          & 91.8               & 95.3  \\
        S4D-LegS    \citep{gu2022on}\tnote{1,2}                  & 1.17M / 191K          & 93.6               & 95.8             \\
        S4D-Inv    \citep{gu2022on}\tnote{1,2}                   & 1.17M / 191K          & 93.4               & 96.2             \\
        S4D-Lin    \citep{gu2022on}\tnote{1,2}                   & 1.17M / 191K          & 93.4               & \underline{96.3}       \\
        Liquid-S4    \citep{hasani2023liquid}\tnote{1,2}         & \underline{875K} / \underline{20K}         & \textbf{96.8}\tnote{1}  & -    \\ 
        S5         \citep{smith2023simplified}\tnote{1,2}        & 1.07M / \textbf{4K}     & -                & \textbf{96.5}    \\ 
        \midrule
        mGRADE                                        & \textbf{773K} / 78K    & \underline{94.7}    & - \\
        \bottomrule
        \end{tabular}\begin{tablenotes}
           \item [1] Uses complex numbers in the recurrence.
           \item [2] Parameter numbers extracted from the official publications.
         \end{tablenotes}
    \end{threeparttable}
    }
\end{table}

\begin{table}[h]
\centering
\caption{\textbf{Ablation of \gls{mgrade}'s component on \gls{lra}.} We scale layer width $H$ to match parameters. $\times$ denotes chance level.}
\label{tab:lraablation_compact}
\vspace{0.1cm}
\resizebox{\columnwidth}{!}{   
    \begin{tabular}{@{}lcccccc@{}}
    \toprule
         & \multicolumn{2}{c}{ListOps} & \multicolumn{2}{c}{Image} & \multicolumn{2}{c}{Pathfinder} \\ 
        \cmidrule(lr){2-3} \cmidrule(lr){4-5} \cmidrule(lr){6-7}
        Model & Acc & Params. / Buff. & Acc & Params. / Buff. & Acc & Params. / Buff. \\ 
    \midrule
        TCN    & 39.6 & 180K/16K  & 85.3 & 2.9M/880K & $\times$ & 3.5M/920K \\ 
        DCLS   & 43.8 & 164K/12K  & 86.2 & 2.1M/880K & $\times$ & 2.1M/960K \\ 
        minGRU & \textbf{62.5} & 160K/768 & 66.0 & 2.8M/3K & $\times$ & 2.5M/3K \\ 
        mGRADE & 61.9 & 160K/12K  & \textbf{87.1} & 2.7M/770K & \textbf{94.9} & 2.3M/770K \\ 
    \bottomrule
    \end{tabular}
}
\end{table}

\section{Related Works} \label{sec:relatedwork}
\textbf{Gated Recurrent Models} Gated \glspl{rnn}, notably the \glspl{lstm} and \glspl{gru}, alleviate vanishing-gradient effects \citep{bengio1994learning-vanishing, hochreiter1997long-lstm} through learned gating mechanisms but need to be trained sequentially and are therefore inefficient for very long sequences.
Removing the hidden state dependency of the update gate enables parallel training through a prefix scan \citep{blelloch1990prefix}, yielding models that are very similar or equivalent to the gated recurrence used by \gls{mgrade} \citep{bradbury2017quasirecurrent,martin2018parallelizing, feng2025were}.
\Glspl{hgrn} extends such parallelizable gated \glspl{rnn} with complex-valued parameters and a hierarchical gating bias, encouraging hierarchical processing of time scales from fast to slow \citep{qin2023HGRN}, but this enforced low-frequency bias can hinder high-frequency feature detection.
\gls{mgrade}'s convolution with learnable spacings addresses this limitation (see Appendix \cref{sec:appendix_highfreqtask}). 

\textbf{Linear Recurrent Models} Linear \glspl{rnn} allow for parallel training by using fully linear transitions rather than gating.
Early investigations into linear \glspl{rnn} \citep{mozer1993neural, mikolov2012context-rnnlm} have converged in linear time-invariant \glspl{ssm}, such as S4 \citep{gu2022efficiently}, \gls{lru} \citep{orvieto2023resurrecting-lru}, and \gls{h3} \citep{dao2022hungry-h3}.
While easily parallelizable, \glspl{ssm} suffer from a low-frequency bias \citep{yu2025tuning} and time-invariant \glspl{ssm} lack \textit{selectivity}, limiting expressivity over long-range dependencies (see Section \cref{sec:flipflop}).
To address this, Mamba \citep{gu2024mamba} reintroduces input-dependent gating, reaching \gls{sota} performance on language modeling benchmarks.
However, Mamba performs worse than previous \glspl{ssm} on long-range dependencies such as those found in the \gls{lra} benchmark \cite{beck2024xlstm, soydan2024s7}.
In terms of memory costs, \glspl{ssm} use complex-valued parameters and require high numerical precision, reducing hardware compatibility and quantizability \citep{zhao2025quantizingsmallscalestatespacemodels}.
In contrast, the \gls{mgrade}'s gated recurrent component has been successfully adapted for low-precision embedded deployment \citep{billaudelle2025minimalist}.

\textbf{Temporal Convolution Models}
\Glspl{tcn} model temporal dependencies within a fixed-length receptive field using causal 1D convolutions over time \citep{waibel1989phoneme}.
Dilated convolutions efficiently increase the temporal receptive field by adding regular spaces between kernel elements, which Wavenet exponentially expands with layer depth \citep{oord2016wavenet}.
However, fixed regular spacings miss information at irregular intervals common in real-world signals \citep{george1997speechanalysis}.
To address this, \gls{dcls} replaces fixed spacings with learnable ones, increasing performance and flexibility in temporal classification with \gls{snn}s \citep{hammouamri2024learning}.
However, \gls{dcls} still buffers inputs, scaling memory with kernel length.
In global convolutional networks, such as SGConv \citep{li2023what-sgconv} and MRConv \citep{cunningham2024reparameterized}, kernel length (and therefore memory cost) adaptively scales with sequence length, effectively maximizing buffer size.
\gls{mgrade}'s learnable spacings use a fixed maximum delay length resulting in a constant activation buffer size, while relying on the gated recurrence to capture longer dependencies.
Similar fixed-length convolutions have been successfully adapted to hardware, notably in DenRAM and Chameleon \citep{dagostino2024denram, blanken2025chameleon}.

\textbf{Combining Convolutions and Recurrence} 
Several recurrent models \citep{bradbury2017quasirecurrent,dao2022hungry-h3, beck2024xlstm} add convolutions to their architecture, yielding consistent performance improvements.
However, the distinct functional contributions of these components, particularly with respect to their intrinsic timescales, remain underexplored prior to this work.
Importantly, none of these works use learnable temporal spacings which are critical to \gls{mgrade}'s theoretical and empirical capabilities.

\section{Conclusion}
We present \gls{mgrade}, a hybrid-memory architecture designed for real-time multi-timescale sequence processing with edge-compatible memory efficiency.
Our design is grounded in formal proofs and experimental evidence that demonstrate \gls{mgrade}'s capacity to model not only short-term dynamics, but also long-range dependencies.
We characterize the functional complementarity of each component: the learnable spacings in the convolution are a short-term cache for delayed inputs, enabling parameter-efficient reconstruction of dynamical system geometries from partial observations, while the gated recurrent component maintains a selectively compressed long-term history of the input. 

We support these theoretical arguments with an extensive empirical evaluation of \gls{mgrade} across the multi-timescale \gls{lra} and raw audio \gls{gsc} tasks.
The results show that \gls{mgrade} substantially reduces the memory footprint compared to the \gls{sota} models, while maintaining highly competitive performance.
This highlights the potential of \gls{mgrade} for large-scale, real-time sequence modeling on resource-constrained embedded systems.

\section*{Impact Statement}
This paper presents work whose goal is to advance the field
of Machine Learning. There are many potential societal
consequences of our work, none which we feel must be
specifically highlighted here.

\bibliography{icml2026/mybib.bib}

\begin{thebibliography}{72}
\providecommand{\natexlab}[1]{#1}
\providecommand{\url}[1]{\texttt{#1}}
\expandafter\ifx\csname urlstyle\endcsname\relax
  \providecommand{\doi}[1]{doi: #1}\else
  \providecommand{\doi}{doi: \begingroup \urlstyle{rm}\Url}\fi

\bibitem[{Advanced Micro Devices, Inc.}(2024)]{versaledgeai_2026}
{Advanced Micro Devices, Inc.}
\newblock Versal ai edge series gen 2 product overview, 2024.
\newblock [Online]. Available: https://www.amd.com/en/products/adaptive-socs-and-fpgas/versal/gen2/ai-edge-series.html.

\bibitem[Beck et~al.(2024)Beck, P{\"o}ppel, Spanring, Auer, Prudnikova, Kopp, Klambauer, Brandstetter, and Hochreiter]{beck2024xlstm}
Beck, M., P{\"o}ppel, K., Spanring, M., Auer, A., Prudnikova, O., Kopp, M.~K., Klambauer, G., Brandstetter, J., and Hochreiter, S.
\newblock x{LSTM}: Extended long short-term memory.
\newblock In \emph{The Thirty-eighth Annual Conference on Neural Information Processing Systems}, 2024.
\newblock URL \url{https://openreview.net/forum?id=ARAxPPIAhq}.

\bibitem[Bengio et~al.(1994)Bengio, Simard, and Frasconi]{bengio1994learning-vanishing}
Bengio, Y., Simard, P., and Frasconi, P.
\newblock Learning long-term dependencies with gradient descent is difficult.
\newblock \emph{IEEE Transactions on Neural Networks}, 5\penalty0 (2):\penalty0 157--166, 1994.
\newblock \doi{10.1109/72.279181}.

\bibitem[Billaudelle et~al.(2025)Billaudelle, Kriener, Moro, Torchet, and Payvand]{billaudelle2025minimalist}
Billaudelle, S., Kriener, L., Moro, F., Torchet, T., and Payvand, M.
\newblock Minimalist: switched-capacitor circuits for efficient in-memory computation of gated recurrent units, 2025.
\newblock URL \url{https://arxiv.org/abs/2505.08599}.

\bibitem[Blelloch(1990)]{blelloch1990prefix}
Blelloch, G.~E.
\newblock Prefix sums and their applications.
\newblock Technical Report CMU-CS-90-190, School of Computer Science, Carnegie Mellon University, November 1990.

\bibitem[Bradbury et~al.(2017)Bradbury, Merity, Xiong, and Socher]{bradbury2017quasirecurrent}
Bradbury, J., Merity, S., Xiong, C., and Socher, R.
\newblock Quasi-recurrent neural networks.
\newblock In \emph{International Conference on Learning Representations}, 2017.
\newblock URL \url{https://openreview.net/forum?id=H1zJ-v5xl}.

\bibitem[Cho et~al.(2014)Cho, van Merri{\"e}nboer, Gulcehre, Bahdanau, Bougares, Schwenk, and Bengio]{cho2014learning-gru}
Cho, K., van Merri{\"e}nboer, B., Gulcehre, C., Bahdanau, D., Bougares, F., Schwenk, H., and Bengio, Y.
\newblock Learning phrase representations using {RNN} encoder{--}decoder for statistical machine translation.
\newblock In Moschitti, A., Pang, B., and Daelemans, W. (eds.), \emph{Proceedings of the 2014 Conference on Empirical Methods in Natural Language Processing ({EMNLP})}, pp.\  1724--1734, Doha, Qatar, October 2014. Association for Computational Linguistics.
\newblock \doi{10.3115/v1/D14-1179}.
\newblock URL \url{https://aclanthology.org/D14-1179/}.

\bibitem[Chung et~al.(2014)Chung, Gulcehre, Cho, and Bengio]{chung2014empirical-gru}
Chung, J., Gulcehre, C., Cho, K., and Bengio, Y.
\newblock Empirical evaluation of gated recurrent neural networks on sequence modeling.
\newblock In \emph{NIPS 2014 Workshop on Deep Learning, December 2014}, 2014.

\bibitem[Covi et~al.(2021)Covi, Donati, Liang, Kappel, Heidari, Payvand, and Wang]{covi2021edgecompute}
Covi, E., Donati, E., Liang, X., Kappel, D., Heidari, H., Payvand, M., and Wang, W.
\newblock Adaptive extreme edge computing for wearable devices.
\newblock \emph{Frontiers in Neuroscience}, Volume 15 - 2021, 2021.
\newblock ISSN 1662-453X.
\newblock \doi{10.3389/fnins.2021.611300}.
\newblock URL \url{https://www.frontiersin.org/journals/neuroscience/articles/10.3389/fnins.2021.611300}.

\bibitem[Cramer et~al.(2020)Cramer, Stradmann, Schemmel, and Zenke]{cramer2020heidelberg}
Cramer, B., Stradmann, Y., Schemmel, J., and Zenke, F.
\newblock The heidelberg spiking data sets for the systematic evaluation of spiking neural networks.
\newblock \emph{IEEE Transactions on Neural Networks and Learning Systems}, 33\penalty0 (7):\penalty0 2744--2757, 2020.

\bibitem[Cunningham et~al.(2024)Cunningham, Giannone, Zhang, and Deisenroth]{cunningham2024reparameterized}
Cunningham, H.~J., Giannone, G., Zhang, M., and Deisenroth, M.~P.
\newblock Reparameterized multi-resolution convolutions for long sequence modelling.
\newblock In \emph{The Thirty-eighth Annual Conference on Neural Information Processing Systems}, 2024.
\newblock URL \url{https://openreview.net/forum?id=RwgNbIpCpk}.

\bibitem[Dao et~al.(2022)Dao, Fu, Saab, Thomas, Rudra, and R{\'{e}}]{dao2022hungry-h3}
Dao, T., Fu, D.~Y., Saab, K.~K., Thomas, A.~W., Rudra, A., and R{\'{e}}, C.
\newblock Hungry hungry hippos: Towards language modeling with state space models.
\newblock \emph{CoRR}, abs/2212.14052, 2022.
\newblock \doi{10.48550/arXiv.2212.14052}.
\newblock URL \url{https://doi.org/10.48550/arXiv.2212.14052}.

\bibitem[den Blanken \& Frenkel(2025)den Blanken and Frenkel]{blanken2025chameleon}
den Blanken, D. and Frenkel, C.
\newblock Chameleon: A matmul-free temporal convolutional network accelerator for end-to-end few-shot and continual learning from sequential data, 2025.

\bibitem[D’agostino et~al.(2024)D’agostino, Moro, Torchet, Demira{\u{g}}, Grenouillet, Castellani, Indiveri, Vianello, and Payvand]{dagostino2024denram}
D’agostino, S., Moro, F., Torchet, T., Demira{\u{g}}, Y., Grenouillet, L., Castellani, N., Indiveri, G., Vianello, E., and Payvand, M.
\newblock Denram: neuromorphic dendritic architecture with rram for efficient temporal processing with delays.
\newblock \emph{Nature communications}, 15\penalty0 (1):\penalty0 3446, 2024.

\bibitem[Feng et~al.(2025)Feng, Tung, Ahmed, Bengio, and Hajimirsadeghi]{feng2025were}
Feng, L., Tung, F., Ahmed, M.~O., Bengio, Y., and Hajimirsadeghi, H.
\newblock Were {RNN}s all we needed?, 2025.
\newblock URL \url{https://openreview.net/forum?id=GrmFFxGnOR}.

\bibitem[George \& Smith(1997)George and Smith]{george1997speechanalysis}
George, E. and Smith, M.
\newblock Speech analysis/synthesis and modification using an analysis-by-synthesis/overlap-add sinusoidal model.
\newblock \emph{IEEE Transactions on Speech and Audio Processing}, 5\penalty0 (5):\penalty0 389--406, 1997.
\newblock \doi{10.1109/89.622558}.

\bibitem[Gers et~al.(2000)Gers, Schmidhuber, and Cummins]{gers2000learning-lstm}
Gers, F.~A., Schmidhuber, J., and Cummins, F.~A.
\newblock Learning to forget: Continual prediction with {LSTM}.
\newblock \emph{Neural Comput.}, 12\penalty0 (10):\penalty0 2451--2471, 2000.
\newblock \doi{10.1162/089976600300015015}.
\newblock URL \url{https://doi.org/10.1162/089976600300015015}.

\bibitem[Gilpin(2024)]{gilpin2024dysts}
Gilpin, W.
\newblock dysts: A chaotic systems simulation library, 2024.
\newblock URL \url{https://github.com/williamgilpin/dysts}.

\bibitem[{Google LLC}(2021)]{googlecoral_2026}
{Google LLC}.
\newblock Google edge tpu accelerator module datasheet.
\newblock Version 1.4, 2021.
\newblock [Online]. Available: https://gweb-coral-full.uc.r.appspot.com/docs/module/datasheet/.

\bibitem[Gu \& Dao(2024)Gu and Dao]{gu2024mamba}
Gu, A. and Dao, T.
\newblock Mamba: Linear-time sequence modeling with selective state spaces.
\newblock In \emph{First Conference on Language Modeling}, 2024.
\newblock URL \url{https://openreview.net/forum?id=tEYskw1VY2}.

\bibitem[Gu et~al.(2020)Gu, Gulcehre, Paine, Hoffman, and Pascanu]{gu2020improving-ugi}
Gu, A., Gulcehre, C., Paine, T., Hoffman, M., and Pascanu, R.
\newblock Improving the gating mechanism of recurrent neural networks.
\newblock In \emph{Proceedings of the 37th International Conference on Machine Learning}, ICML'20. JMLR.org, 2020.

\bibitem[Gu et~al.(2022{\natexlab{a}})Gu, Goel, Gupta, and R{\'e}]{gu2022on}
Gu, A., Goel, K., Gupta, A., and R{\'e}, C.
\newblock On the parameterization and initialization of diagonal state space models.
\newblock In Oh, A.~H., Agarwal, A., Belgrave, D., and Cho, K. (eds.), \emph{Advances in Neural Information Processing Systems}, 2022{\natexlab{a}}.
\newblock URL \url{https://openreview.net/forum?id=yJE7iQSAep}.

\bibitem[Gu et~al.(2022{\natexlab{b}})Gu, Goel, and R\'e]{gu2022efficiently}
Gu, A., Goel, K., and R\'e, C.
\newblock Efficiently modeling long sequences with structured state spaces.
\newblock In \emph{The International Conference on Learning Representations ({ICLR})}, 2022{\natexlab{b}}.

\bibitem[Gupta et~al.(2022)Gupta, Gu, and Berant]{gupta2022diagonal}
Gupta, A., Gu, A., and Berant, J.
\newblock Diagonal state spaces are as effective as structured state spaces.
\newblock In Oh, A.~H., Agarwal, A., Belgrave, D., and Cho, K. (eds.), \emph{Advances in Neural Information Processing Systems}, 2022.
\newblock URL \url{https://openreview.net/forum?id=RjS0j6tsSrf}.

\bibitem[Göltz et~al.(2025)Göltz, Weber, Kriener, Billaudelle, Lake, Schemmel, Payvand, and Petrovici]{goeltz2025delgrad}
Göltz, J., Weber, J., Kriener, L., Billaudelle, S., Lake, P., Schemmel, J., Payvand, M., and Petrovici, M.~A.
\newblock Delgrad: exact event-based gradients for training delays and weights on spiking neuromorphic hardware.
\newblock \emph{Nature Communications}, 16:\penalty0 8245, 2025.
\newblock \doi{10.1038/s41467-025-63120-y}.

\bibitem[Hammouamri et~al.(2024)Hammouamri, Khalfaoui-Hassani, and Masquelier]{hammouamri2024learning}
Hammouamri, I., Khalfaoui-Hassani, I., and Masquelier, T.
\newblock Learning delays in spiking neural networks using dilated convolutions with learnable spacings.
\newblock In \emph{The Twelfth International Conference on Learning Representations}, 2024.
\newblock URL \url{https://openreview.net/forum?id=4r2ybzJnmN}.

\bibitem[Hasani et~al.(2023)Hasani, Lechner, Wang, Chahine, Amini, and Rus]{hasani2023liquid}
Hasani, R., Lechner, M., Wang, T.-H., Chahine, M., Amini, A., and Rus, D.
\newblock Liquid structural state-space models.
\newblock In \emph{The Eleventh International Conference on Learning Representations}, 2023.
\newblock URL \url{https://openreview.net/forum?id=g4OTKRKfS7R}.

\bibitem[Hassani et~al.(2023)Hassani, Pellegrini, and Masquelier]{hassani2023dilated-dcls}
Hassani, I.~K., Pellegrini, T., and Masquelier, T.
\newblock Dilated convolution with learnable spacings.
\newblock In \emph{The Eleventh International Conference on Learning Representations}, 2023.
\newblock URL \url{https://openreview.net/forum?id=Q3-1vRh3HOA}.

\bibitem[Hochreiter \& Schmidhuber(1997)Hochreiter and Schmidhuber]{hochreiter1997long-lstm}
Hochreiter, S. and Schmidhuber, J.
\newblock Long short-term memory.
\newblock \emph{Neural computation}, 9\penalty0 (8):\penalty0 1735--1780, 1997.

\bibitem[Howitt et~al.(2024)Howitt, Nair, Dods, and Hopkins]{howitt2024}
Howitt, K., Nair, S., Dods, A., and Hopkins, R.~M.
\newblock Generalizations across filler-gap dependencies in neural language models.
\newblock In Barak, L. and Alikhani, M. (eds.), \emph{Proceedings of the 28th Conference on Computational Natural Language Learning}, pp.\  269--279, Miami, FL, USA, November 2024. Association for Computational Linguistics.
\newblock \doi{10.18653/v1/2024.conll-1.21}.
\newblock URL \url{https://aclanthology.org/2024.conll-1.21/}.

\bibitem[Hyndman \& Koehler(2006)Hyndman and Koehler]{hyndmanMASE2006}
Hyndman, R.~J. and Koehler, A.~B.
\newblock Another look at measures of forecast accuracy.
\newblock \emph{International Journal of Forecasting}, 22\penalty0 (4):\penalty0 679--688, 2006.
\newblock ISSN 0169-2070.
\newblock \doi{https://doi.org/10.1016/j.ijforecast.2006.03.001}.
\newblock URL \url{https://www.sciencedirect.com/science/article/pii/S0169207006000239}.

\bibitem[Jing et~al.(2019)Jing, Gulcehre, Peurifoy, Shen, Tegmark, Soljacic, and Bengio]{tegmark2019orthogonalru}
Jing, L., Gulcehre, C., Peurifoy, J., Shen, Y., Tegmark, M., Soljacic, M., and Bengio, Y.
\newblock Gated orthogonal recurrent units: On learning to forget.
\newblock \emph{Neural Computation}, 31\penalty0 (4):\penalty0 765--783, 04 2019.
\newblock ISSN 0899-7667.
\newblock \doi{10.1162/neco_a_01174}.
\newblock URL \url{https://doi.org/10.1162/neco\_a\_01174}.

\bibitem[Khalfaoui-Hassani et~al.(2023)Khalfaoui-Hassani, Pellegrini, and Masquelier]{hassani2023dcls2}
Khalfaoui-Hassani, I., Pellegrini, T., and Masquelier, T.
\newblock Dilated convolution with learnable spacings: beyond bilinear interpolation.
\newblock In \emph{Differentiable Almost Everything Workshop of the 40-th International Conference on Machine Learning}, 2023.
\newblock URL \url{https://openreview.net/forum?id=4r2ybzJnmN}.

\bibitem[Kingma \& Ba(2017)Kingma and Ba]{kingma2017adam}
Kingma, D.~P. and Ba, J.
\newblock Adam: A method for stochastic optimization, 2017.
\newblock URL \url{https://arxiv.org/abs/1412.6980}.

\bibitem[Li et~al.(2023)Li, Cai, Zhang, Chen, and Dey]{li2023what-sgconv}
Li, Y., Cai, T., Zhang, Y., Chen, D., and Dey, D.
\newblock What makes convolutional models great on long sequence modeling?
\newblock In \emph{The Eleventh International Conference on Learning Representations}, 2023.
\newblock URL \url{https://openreview.net/forum?id=TGJSPbRpJX-}.

\bibitem[Liu et~al.(2023)Liu, Ash, Goel, Krishnamurthy, and Zhang]{liu2023attentionglitch}
Liu, B., Ash, J.~T., Goel, S., Krishnamurthy, A., and Zhang, C.
\newblock Exposing attention glitches with flip-flop language modeling.
\newblock In \emph{Thirty-seventh Conference on Neural Information Processing Systems}, 2023.
\newblock URL \url{https://openreview.net/forum?id=VzmpXQAn6E}.

\bibitem[Lorenz(1963)]{Lorenz1963}
Lorenz, E.~N.
\newblock Deterministic nonperiodic flow.
\newblock \emph{Journal of the Atmospheric Sciences}, 20:\penalty0 130–141, 1963.

\bibitem[Loshchilov \& Hutter(2019)Loshchilov and Hutter]{loshchilov2018decoupled-adamw}
Loshchilov, I. and Hutter, F.
\newblock Decoupled weight decay regularization.
\newblock In \emph{International Conference on Learning Representations}, 2019.
\newblock URL \url{https://openreview.net/forum?id=Bkg6RiCqY7}.

\bibitem[Maass \& Schmitt(1999)Maass and Schmitt]{maas1999temporalcoding}
Maass, W. and Schmitt, M.
\newblock On the complexity of learning for spiking neurons with temporal coding.
\newblock \emph{Information and Computation}, 153\penalty0 (1):\penalty0 26--46, 1999.
\newblock ISSN 0890-5401.
\newblock \doi{https://doi.org/10.1006/inco.1999.2806}.
\newblock URL \url{https://www.sciencedirect.com/science/article/pii/S0890540199928067}.

\bibitem[Martin \& Cundy(2018)Martin and Cundy]{martin2018parallelizing}
Martin, E. and Cundy, C.
\newblock Parallelizing linear recurrent neural nets over sequence length.
\newblock In \emph{6th International Conference on Learning Representations, {ICLR} 2018, Vancouver, BC, Canada, April 30 - May 3, 2018, Conference Track Proceedings}. OpenReview.net, 2018.
\newblock URL \url{https://openreview.net/forum?id=HyUNwulC-}.

\bibitem[Mikolov \& Zweig(2012)Mikolov and Zweig]{mikolov2012context-rnnlm}
Mikolov, T. and Zweig, G.
\newblock Context dependent recurrent neural network language model.
\newblock In \emph{2012 IEEE Spoken Language Technology Workshop (SLT)}, pp.\  234--239, 2012.
\newblock \doi{10.1109/SLT.2012.6424228}.

\bibitem[Miralles-González et~al.(2025)Miralles-González, Huertas-Tato, Martín, and Camacho]{miralles2025lralocality}
Miralles-González, P., Huertas-Tato, J., Martín, A., and Camacho, D.
\newblock On the locality bias and results in the long range arena, 2025.
\newblock URL \url{https://arxiv.org/abs/2501.14850}.

\bibitem[Mozer(1993)]{mozer1993neural}
Mozer, M.
\newblock Neural net architectures for temporal sequence processing.
\newblock \emph{Santa Fe Institute Studies in The Sciences of Complexity}, 15:\penalty0 243--243, 03 1993.

\bibitem[Mutlu et~al.(2025)Mutlu, Olgun, and Yüksel]{mutlu2025memorycentric}
Mutlu, O., Olgun, A., and Yüksel, I.~E.
\newblock Memory-centric computing: Solving computing’s memory problem.
\newblock In \emph{2025 IEEE International Memory Workshop (IMW)}, pp.\  1--4, 2025.
\newblock \doi{10.1109/IMW61990.2025.11026935}.

\bibitem[Nangia \& Bowman(2018)Nangia and Bowman]{nangia2018listops}
Nangia, N. and Bowman, S.
\newblock {L}ist{O}ps: A diagnostic dataset for latent tree learning.
\newblock In Cordeiro, S.~R., Oraby, S., Pavalanathan, U., and Rim, K. (eds.), \emph{Proceedings of the 2018 Conference of the North {A}merican Chapter of the Association for Computational Linguistics: Student Research Workshop}, pp.\  92--99, New Orleans, Louisiana, USA, June 2018. Association for Computational Linguistics.
\newblock \doi{10.18653/v1/N18-4013}.
\newblock URL \url{https://aclanthology.org/N18-4013/}.

\bibitem[Orvieto et~al.(2023)Orvieto, Smith, Gu, Fernando, G{\"{u}}l{\c{c}}ehre, Pascanu, and De]{orvieto2023resurrecting-lru}
Orvieto, A., Smith, S.~L., Gu, A., Fernando, A., G{\"{u}}l{\c{c}}ehre, {\c{C}}., Pascanu, R., and De, S.
\newblock Resurrecting recurrent neural networks for long sequences.
\newblock \emph{CoRR}, abs/2303.06349, 2023.
\newblock \doi{10.48550/arXiv.2303.06349}.
\newblock URL \url{https://doi.org/10.48550/arXiv.2303.06349}.

\bibitem[Ostrow et~al.(2024)Ostrow, Eisen, and Fiete]{Ostrow2024}
Ostrow, M., Eisen, A.~J., and Fiete, I.~R.
\newblock Delay embedding theory of neural sequence models.
\newblock In \emph{ICML 2024 Workshop on Mechanistic Interpretability}, 2024.
\newblock URL \url{https://openreview.net/forum?id=wew3SpwIqr}.

\bibitem[Pascanu et~al.(2013)Pascanu, Mikolov, and Bengio]{pascanu2013difficulty}
Pascanu, R., Mikolov, T., and Bengio, Y.
\newblock On the difficulty of training recurrent neural networks.
\newblock In Dasgupta, S. and McAllester, D. (eds.), \emph{Proceedings of the 30th International Conference on Machine Learning}, volume~28 of \emph{Proceedings of Machine Learning Research}, pp.\  1310--1318, Atlanta, Georgia, USA, 17--19 Jun 2013. PMLR.
\newblock URL \url{https://proceedings.mlr.press/v28/pascanu13.html}.

\bibitem[Qin et~al.(2023)Qin, Yang, and Zhong]{qin2023HGRN}
Qin, Z., Yang, S., and Zhong, Y.
\newblock Hierarchically gated recurrent neural network for sequence modeling.
\newblock In Oh, A., Naumann, T., Globerson, A., Saenko, K., Hardt, M., and Levine, S. (eds.), \emph{Advances in Neural Information Processing Systems 36: Annual Conference on Neural Information Processing Systems 2023, NeurIPS 2023, New Orleans, LA, USA, December 10 - 16, 2023}, 2023.

\bibitem[Rahaman et~al.(2019)Rahaman, Baratin, Arpit, Draxler, Lin, Hamprecht, Bengio, and Courville]{rahaman2019spectralbias}
Rahaman, N., Baratin, A., Arpit, D., Draxler, F., Lin, M., Hamprecht, F., Bengio, Y., and Courville, A.
\newblock On the spectral bias of neural networks.
\newblock In Chaudhuri, K. and Salakhutdinov, R. (eds.), \emph{Proceedings of the 36th International Conference on Machine Learning}, volume~97 of \emph{Proceedings of Machine Learning Research}, pp.\  5301--5310. PMLR, 09--15 Jun 2019.
\newblock URL \url{https://proceedings.mlr.press/v97/rahaman19a.html}.

\bibitem[Sarrof et~al.(2024)Sarrof, Veitsman, and Hahn]{sarrof2024}
Sarrof, Y., Veitsman, Y., and Hahn, M.
\newblock The expressive capacity of state space models: A formal language perspective.
\newblock In \emph{The Thirty-eighth Annual Conference on Neural Information Processing Systems}, 2024.
\newblock URL \url{https://openreview.net/forum?id=eV5YIrJPdy}.

\bibitem[Schlag et~al.(2021)Schlag, Irie, and Schmidhuber]{schlag2021lineartransformerssecretlyfast}
Schlag, I., Irie, K., and Schmidhuber, J.
\newblock Linear transformers are secretly fast weight programmers, 2021.
\newblock URL \url{http:// proceedings.mlr.press/v139/schlag21a.html.}

\bibitem[Sch{\"o}ne et~al.(2024)Sch{\"o}ne, Sushma, Zhuge, Mayr, Subramoney, and Kappel]{schone2024scalable}
Sch{\"o}ne, M., Sushma, N.~M., Zhuge, J., Mayr, C., Subramoney, A., and Kappel, D.
\newblock Scalable event-by-event processing of neuromorphic sensory signals with deep state-space models.
\newblock In \emph{2024 International Conference on Neuromorphic Systems (ICONS)}, pp.\  124--131. IEEE, 2024.

\bibitem[Schutter(2000)]{schutter2000minimalssm}
Schutter, B.
\newblock Minimal state-space realization in linear system theory: an overview.
\newblock \emph{Journal of Computational and Applied Mathematics}, 121\penalty0 (1):\penalty0 331--354, 2000.
\newblock ISSN 0377-0427.
\newblock \doi{https://doi.org/10.1016/S0377-0427(00)00341-1}.
\newblock URL \url{https://www.sciencedirect.com/science/article/pii/S0377042700003411}.

\bibitem[Shi et~al.(2023)Shi, Chen, Misra, Scales, Dohan, Chi, Sch\"{a}rli, and Zhou]{shi2023}
Shi, F., Chen, X., Misra, K., Scales, N., Dohan, D., Chi, E., Sch\"{a}rli, N., and Zhou, D.
\newblock Large language models can be easily distracted by irrelevant context.
\newblock In \emph{Proceedings of the 40th International Conference on Machine Learning}, ICML'23. JMLR.org, 2023.

\bibitem[{SiMa Technologies, Inc.}(2024)]{simamlsoc_2024}
{SiMa Technologies, Inc.}
\newblock Machine learning system on chip (mlsoc) product brief.
\newblock Rev. May 20, 2024, 2024.
\newblock URL \url{https://sima.ai/wp-content/uploads/2024/06/SiMa_MLSoC_ProductBrief_5.20.24.pdf}.

\bibitem[Smith et~al.(2023)Smith, Warrington, and Linderman]{smith2023simplified}
Smith, J.~T., Warrington, A., and Linderman, S.
\newblock Simplified state space layers for sequence modeling.
\newblock In \emph{The Eleventh International Conference on Learning Representations}, 2023.
\newblock URL \url{https://openreview.net/forum?id=Ai8Hw3AXqks}.

\bibitem[Soydan et~al.(2024)Soydan, Zubić, Messikommer, Mishra, and Scaramuzza]{soydan2024s7}
Soydan, T., Zubić, N., Messikommer, N., Mishra, S., and Scaramuzza, D.
\newblock S7: Selective and simplified state space layers for sequence modeling.
\newblock \emph{arXiv}, 2024.
\newblock \doi{10.48550/arxiv.2410.03464}.

\bibitem[{STMicroelectronics NV}(2025)]{stm32n6_2025}
{STMicroelectronics NV}.
\newblock Stm32n6x5xx stm32n6x7xx mcu datasheet.
\newblock DS14791 Rev 9, 2025.
\newblock URL \url{https://www.st.com/resource/en/datasheet/stm32n657a0.pdf}.

\bibitem[Strogatz(2015)]{strogatz2015dynamics}
Strogatz, S.
\newblock \emph{Chemical chaos and attractor reconstruction}, chapter 12.4.
\newblock CRC Press, 2015.

\bibitem[{Syntiant Corp.}(2024)]{syntiantndp250_2024}
{Syntiant Corp.}
\newblock Neural decision processor ndp250 datasheet, 2024.
\newblock URL \url{https://www.syntiant.com/ndp250#data_sheet}.

\bibitem[Takens(1981)]{Takens1981}
Takens, F.
\newblock Detecting strange attractors in turbulence.
\newblock \emph{Dynamical Systems and Turbulence, Lecture Notes in Mathematics}, 898:\penalty0 366–381, 1981.

\bibitem[Tay et~al.(2021)Tay, Dehghani, Abnar, Shen, Bahri, Pham, Rao, Yang, Ruder, and Metzler]{tay2021lra}
Tay, Y., Dehghani, M., Abnar, S., Shen, Y., Bahri, D., Pham, P., Rao, J., Yang, L., Ruder, S., and Metzler, D.
\newblock Long range arena : {A} benchmark for efficient transformers.
\newblock In \emph{9th International Conference on Learning Representations, {ICLR} 2021, Virtual Event, Austria, May 3-7, 2021}. OpenReview.net, 2021.
\newblock URL \url{https://openreview.net/forum?id=qVyeW-grC2k}.

\bibitem[van~den Oord et~al.(2016)van~den Oord, Dieleman, Zen, Simonyan, Vinyals, Graves, Kalchbrenner, Senior, and Kavukcuoglu]{oord2016wavenet}
van~den Oord, A., Dieleman, S., Zen, H., Simonyan, K., Vinyals, O., Graves, A., Kalchbrenner, N., Senior, A., and Kavukcuoglu, K.
\newblock Wavenet: A generative model for raw audio, 2016.
\newblock URL \url{https://arxiv.org/abs/1609.03499}.

\bibitem[Vaswani et~al.(2017)Vaswani, Shazeer, Parmar, Uszkoreit, Jones, Gomez, Kaiser, and Polosukhin]{vaswani2017attention}
Vaswani, A., Shazeer, N., Parmar, N., Uszkoreit, J., Jones, L., Gomez, A.~N., Kaiser, L.~u., and Polosukhin, I.
\newblock Attention is all you need.
\newblock In Guyon, I., Luxburg, U.~V., Bengio, S., Wallach, H., Fergus, R., Vishwanathan, S., and Garnett, R. (eds.), \emph{Advances in Neural Information Processing Systems}, volume~30. Curran Associates, Inc., 2017.
\newblock URL \url{https://proceedings.neurips.cc/paper_files/paper/2017/file/3f5ee243547dee91fbd053c1c4a845aa-Paper.pdf}.

\bibitem[Waibel et~al.(1989)Waibel, Hanazawa, Hinton, Shikano, and Lang]{waibel1989phoneme}
Waibel, A., Hanazawa, T., Hinton, G., Shikano, K., and Lang, K.
\newblock Phoneme recognition using time-delay neural networks.
\newblock \emph{IEEE Transactions on Acoustics, Speech, and Signal Processing}, 37\penalty0 (3):\penalty0 328--339, 1989.
\newblock \doi{10.1109/29.21701}.

\bibitem[Wang et~al.(2025)Wang, Yu, Shen, Guo, Zhou, Zhao, Zhong, Ma, and Zhang]{wang2025spikcommander}
Wang, J., Yu, L., Shen, X., Guo, S., Zhou, C., Zhao, L., Zhong, Y., Ma, Z., and Zhang, Z.
\newblock Spikcommander: A high-performance spiking transformer with multi-view learning for efficient speech command recognition.
\newblock \emph{arXiv preprint arXiv:2511.07883}, 2025.

\bibitem[Warden(2018)]{warden2018speech}
Warden, P.
\newblock Speech {C}ommands: A dataset for limited-vocabulary speech recognition.
\newblock \emph{arXiv preprint arXiv:1804.03209}, 2018.

\bibitem[Wilcox et~al.(2018)Wilcox, Levy, Morita, and Futrell]{wilcox2018fillergap}
Wilcox, E., Levy, R., Morita, T., and Futrell, R.
\newblock What do {RNN} language models learn about filler{--}gap dependencies?
\newblock In Linzen, T., Chrupa{\l}a, G., and Alishahi, A. (eds.), \emph{Proceedings of the 2018 {EMNLP} Workshop {B}lackbox{NLP}: Analyzing and Interpreting Neural Networks for {NLP}}, pp.\  211--221, Brussels, Belgium, November 2018. Association for Computational Linguistics.
\newblock \doi{10.18653/v1/W18-5423}.
\newblock URL \url{https://aclanthology.org/W18-5423/}.

\bibitem[Yu et~al.(2025)Yu, Lyu, Lim, Mahoney, and Erichson]{yu2025tuning}
Yu, A., Lyu, D., Lim, S.~H., Mahoney, M.~W., and Erichson, N.~B.
\newblock Tuning frequency bias of state space models.
\newblock In \emph{The Thirteenth International Conference on Learning Representations}, 2025.
\newblock URL \url{https://openreview.net/forum?id=wkHcXDv7cv}.

\bibitem[Zhao et~al.(2025)Zhao, Torchet, Payvand, Kriener, and Moro]{zhao2025quantizingsmallscalestatespacemodels}
Zhao, L., Torchet, T., Payvand, M., Kriener, L., and Moro, F.
\newblock Quantizing small-scale state-space models for edge ai, 2025.
\newblock URL \url{https://arxiv.org/abs/2506.12480}.

\bibitem[Zucchet(2024)]{Zucchet2024minimalLRU}
Zucchet, N.
\newblock {minimal-LRU}: Unofficial implementation of the linear recurrent unit (lru, orvieto {et al.} 2023).
\newblock \url{https://github.com/NicolasZucchet/minimal‑LRU}, 2024.
\newblock MIT License. Accessed: 2025‑09‑22.

\end{thebibliography}
\bibliographystyle{icml2026/icml2026}

\newpage
\appendix
\onecolumn
\newpage
\appendix
\section*{\Large\textbf{Appendix}}

\counterwithin*{figure}{part}
\counterwithin*{table}{part}
\stepcounter{part}

\renewcommand{\thefigure}{A\arabic{figure}}
\setcounter{figure}{0}
\renewcommand{\thetable}{A\arabic{table}}
\setcounter{table}{0}

\section{Model specification details}

\subsection{Learnable DCLS kernels}\label{sec:appendix_dcls_kernel}

\gls{dcls} was first introduced by \citet{hassani2023dilated-dcls} and enables the spacings between different elements of a convolution kernel to be trained.
In a temporal setting, this is equivalent to learning delays.
\gls{dcls} requires a specific kernel construction parameterized by both a set of weights, $\Omega_d=\{w_0, w_1, ..., w_{K-1}\ | \text{ } w_i \in \mathbb{R}\}$, and a set of corresponding positions, $\Psi_d=\{p_0, p_1, ..., p_{K-1}\ | \text{ }  p_i \in \mathbb{R},\ p_i\leq p_{max}\}$, for every channel $d\leq D$ of the input.
These sets have $K$, the kernel count, elements each, and a maximum possible position (or in time, a maximum delay) $p_{max} = \Gamma$, called the kernel length.
In time, each position is relative to the current timestep, making it equivalent to a transmission delay.

To construct the discrete kernel $k_d$ for one of the input channels, each real-valued position $p_i$ is mapped to the discrete kernel indices $n \leq \Gamma$ via a differentiable interpolation function, $c$.
This enables both the position and weight of the kernel elements to be learned with gradient descent.
The kernel $k_d \in \mathbb{R}^\Gamma$ for a single channel then becomes:
\begin{align}
\ k_d[n] &= \sum_{i=0}^{K-1} w_i \cdot c[n, p_i], \quad \text{with} \quad k_d = \left[k_d[0], k_d[1], ...,k_d[\Gamma - 1]\right]
\label{eq:kernel_dcls_appendix}
\end{align}

As in \citet{hassani2023dcls2}, we use a Gaussian with fixed width $v$ as our interpolation function:
\begin{align}
c[n,p_i] &= \exp{\left[\frac{-1}{2}\left(\frac{n-p_i}{v}\right)^2\right]} \label{eq:interpolation_func}
\end{align}

The \gls{dcls} convolution's output $x_{d,t}$ for each channel $d$ at timestep $t$ is computed as follows:
\begin{align}
x_{d,t}= (u_{d} * k_d)[t] &= \sum_{n=0}^{\Gamma-1} k_d[n] \odot u_d[t-n] 
\label{eq:convolution_sum}
\end{align}

\subsection{Memory footprint}\label{sec:appendix_memoryfootprint}

\paragraph{Scaling}

The memory requirements of mGRADE during inference consist of (1) the model parameters and (2) the activation buffer required for sequential processing.
Since these come with distinct usage patterns and on-chip implementations in embedded systems, we treat them separately as they might necessitate employing different memory technologies.

In terms of the number of parameters, the temporal convolution component of each \gls{mgrade} layer scales with the number of channels (or model dimensionality) $D$ and the number of kernel elements $K$, leading to $\mathcal{O}(D\times K)$ complexity.
In practice, $K$ is significantly lower than $D$ or $\Gamma$.
The gated recurrent component scales with the model dimensionality and the hidden state dimensionality $H$, with $\mathcal{O}(D\times H)$, just like the \gls{mlp} after the gated recurrence.
Assuming that $H$ is proportional to $D$, the overall parameter memory scales as $\mathcal{O}(D^2)$.

In terms of activation buffer, the temporal convolution with learnable delays requires storing input activations for at most $\Gamma$ timesteps (\cref{eq:convolution_sum}).
More precisely, the activation buffer size scales linearly with the model dimensionality $D$, the number of layers $L$, and the kernel length $\Gamma$, yielding a memory complexity of $\mathcal{O}(D\times L \times \Gamma)$.
The gated recurrent component only requires a single hidden state vector per layer (similar to the \gls{mlp}), so assuming the hidden state dimensionality $H$ is proportional to $D$, the overall activation buffer complexity is thus dominated by the temporal convolution.

\paragraph{Calculation}
 
This section presents complete derivations for the memory requirements of each network component during inference. 
We begin by examining the parameter memory footprint, followed by an analysis of buffer memory usage. 

The notation $\text{MemParam}^{\text{component}}$ represents the memory consumption for each component (encoder, convolution, recurrent, MLP, decoder), with subelements categorized as Weights, Bias, and Positions.
Note that the hidden state dimensionality of the recurrent component, $H$, is the model dimensionality scaled by an expansion factor denoted as $e$.

\begin{align}
    \text{MemParam}^{\text{Enc}} &= \text{Weights}^{\text{Encoder}} + \text{Bias}^{\text{Encoder}} = D_{\text{in}} \times D \\
    \text{MemParam}^{\text{Conv}} &= 
    \begin{cases} 
        \text{Weights}^{\text{Conv}} + \text{Positions}^{\text{Conv}} = 2(K \times D) & \text{for mGRADE,} \\
        \text{Weights}^{\text{Conv}} = K \times D & \text{else}. \\
    \end{cases}\\
    \text{MemParam}^{\text{Rec}} &= \text{Weights}_z + \text{Weights}_{\tilde{h}} + \text{Bias}_z + \text{Bias}_{\tilde{h}}  \\
     & = H \times D + H\times D + H + H = 2eD^2 + 2D \\
    \text{MemParam}^{\text{MLP}} &= \text{Weights}^{\text{MLP}} + \text{Bias}^{MLP} \\ &= 2D \times H + 2D + D \times 2D + D = 4eD^2 + 3D\\
    \text{MemParam}^{\text{Norm}} &= 2D \\
    \text{MemParam}^{\text{Dec}} &= \text{Weights}^{\text{Dec}}  = D \times D_{\text{out}} \\
    \text{MemParam}^{\text{network}} &= \text{MemParam}^{\text{Enc}} + \text{ MemParam}^{\text{Dec}} \notag \\ 
    & \quad + L \times (\text{MemParam}^{\text{Conv}} + \text{MemParam}^{\text{Rec}} \\ & \quad \quad \quad \quad + \text{MemParam}^{\text{MLP}} + \text{ MemParam}^{\text{Norm}}) \notag   
\end{align}

Overall, this yields a memory complexity of $\mathcal{O}(L\times e \times D^2+L\times K\times D)$ for parameter storage, with the model dimensionality $D$ dominating.

As explained in \cref{sec:modelspec}, any single convolution layer requires the storage of past inputs to produce an output. The past inputs are stored in an activation buffer whose size scales linearly with the kernel length $\Gamma$. Thus, for an entire network, the buffer requirements for the convolutional components are determined by $D$, $L$, and the kernel length $\Gamma$.
In addition, recurrent models must maintain hidden state activations with dimensionality $H=eD$, further contributing to the required activation buffer.

\begin{align}
    \text{MemBuffer}^{\text{Conv}} &= L \times D \times \Gamma \quad \text{with} \ \
    \Gamma   = 
    \begin{cases} 
        d \times(K-1) & \text{TCN with dilation rate } d, \label{eq:buffer_size} \\
        p_{max} & \text{ mGRADE}.
    \end{cases} \\   
    \text{MemBuffer}^{\text{Rec}} &= L \times H = L \times eD 
\end{align}

For purely recurrent models, the activation buffer size is determined entirely by $\text{MemBuffer}^{\text{Rec}}$ and, for purely convolutional models, by $\text{MemBuffer}^{\text{Conv}}$.
\gls{mgrade}'s final required memory is just the sum of the two:

\begin{align}
\text{MemBuffer}^{\text{network}} &= \text{MemBuffer}^{\text{Conv}} + \text{MemBuffer}^{\text{Rec}}  \\ &= L \times D \times \Gamma + L \times eD \\ &= L \times D \times (\Gamma + e)
\end{align}

Thus, \gls{mgrade}'s activation buffer does not scale with sequence length but instead with layer number, model dimensionality, and kernel length.

Note that we use a default precision of 32 bits for all of our activations and parameters.
To get the total number of bytes of memory required for either parameters or activation buffers, we multiply $\text{MemParam}^{\text{network}}$ or $\text{MemBuffer}^{\text{network}}$ by the bit precision and divide by 8.

\paragraph{Why not parameterize the convolution as a \gls{ssm}?}

Given the equivalence between convolutions and \glspl{ssm} elaborated in \cite{gu2022efficiently}, it is worth clarifying why we parameterize the learnable spacings with an explicit convolutional representation instead of a recurrent \gls{ssm}. 
During training, when the full input sequence is available, the convolution can be efficiently computed via the FFT formulation. 
During inference, a fully instantiated convolution remains more memory-efficient.
Representing an arbitrary convolution with an SSM during inference requires the same number of parameters and buffered activations as parameterizing the convolution directly. 
This is because, in the worst case, an impulse response whose Hankel matrix has rank $\Gamma$ (equal to the kernel length) requires an SSM of state dimension $\Gamma$ \citep{schutter2000minimalssm}.

Thus, while it is theoretically possible to reformulate our convolution with learnable spacings as a fully recurrent SSM, doing so would not provide any memory advantage during inference or training. 
In fact, since we use learnable spacings between kernel elements, a naive SSM reformulation would require a transition matrix of size $\Gamma \times \Gamma$ even though only $K$ kernel weights are actually needed.
For a discussion on the difference between \textit{trainable} parameters and actual \textit{instantiated} parameters during recurrent inference, see \cref{sec:appendix_embedded_memory}.

\section{Theoretical Capabilities of mGRADE}\label{sec:appendix}

\subsection{Dynamics Reconstruction Task}\label{sec:appendix_lorenz_task}

\subsubsection{Proof for mGRADE as a Delay Embedding}\label{sec:appendix_lorenz_proof}

Here we detail the full proof of \cref{theorem:mgrade_delay_embedding}, demonstrating how mGRADE can learn to express a delay embedding.

{\renewcommand{\thetheorem}{\ref{theorem:mgrade_delay_embedding}}
\begin{theorem}[Reconstructing Dynamics through mGRADE’s Delay Embedding] Take a discrete-time dynamical system \(f : \mathcal{M} \to \mathcal{M}\) over a compact manifold \(M\) of dimension \(d\), mapping \(\mathbf{u}_t \in \mathbb{R}^d\) to \(\mathbf{u}_{t+1} \in \mathbb{R}^d\) according to some differentiable and deterministic rule. 
Let \(y : \mathcal{M} \to \mathbb{R}\) be a generic twice-differentiable observation function that deterministically maps any \(\mathbf{u}_t\)  on \(\mathcal{M}\) to a single observable \(y_t \in \mathbb{R}\) at time $t$.
Let \(m \geq 2d + 1\). 
Let a single \gls{mgrade} layer be preceded by a linear projection mapping the input \(y_t\) to \(\mathbf{\hat{y}}_t \in \mathbb{R}^m\) (the encoder) such that \(D = H = m\). 
Assume that \(v \to 0\), with $v$ being the width of mGRADE's interpolation function \(c\).

Then, the hidden state \(\mathbf{h}_t\in \mathbb{R}^m\) of $m$-dimensional mGRADE layer can learn to express a delay embedding in the sense of \citet{Takens1981}, and accordingly can learn to fully reconstruct the original dynamics of \(\mathbf{u}_t \in \mathbb{R}^d\), differing only by a smooth, invertible change of coordinates (a \emph{diffeomorphism}). 
\end{theorem}
\addtocounter{theorem}{-1}
}

\begin{proof}
Let \( f : \mathcal{M} \to \mathcal{M} \) be a discrete-time dynamical system over a compact \( d \)-dimensional manifold \( \mathcal{M} \), and let \( y : \mathcal{M} \to \mathbb{R} \) be a generic twice-differentiable observation function, as defined in Theorem~3. Assume \( m \geq 2d + 1 \). Let the encoder preceding the mGRADE layer map \( y_t \) to a \( m \)-dimensional vector \( \mathbf{\hat{y}}_t = [y_t, y_t, \dots, y_t] \in \mathbb{R}^m \) by replication.

Construct a single-layer mGRADE model as follows:\\

1. Each of the \( m \) channels of mGRADE's temporal convolution kernel is learned to a unique, fixed delay. Specifically, for the \( q \)-th channel, set the kernel's delay to $p_q=(q - 1)\tau$, for some fixed delay step \( \tau \in \mathbb{N} \), letting the interpolation function width \( v \to 0 \) and learning unitary weights. This means that the output of the \( q \)-th channel is equal to \( y_{t - (q - 1)\tau} \).\\

2. Let the update gate parameter matrix \( \mathbf{W}_z \) be such that \( \mathbf{h}_t = \tilde{\mathbf{h}}_t \) at every timestep (by setting \( \mathbf{W}_z \to -\infty \)). Let the hidden projection matrix \( \mathbf{W}_h \in \mathbb{R}^{m \times m} \) be the identity.\\

Under this construction, the hidden state at each time \( t \) becomes:
\[
\mathbf{h}_t = \tilde{\mathbf{h}}_t = [y_t, y_{t - \tau}, y_{t - 2\tau}, \dots, y_{t - (m - 1)\tau}],
\]
which corresponds exactly to the delay vector used by Takens as the delay embedding \citep{Takens1981}.

Because \( m \geq 2d + 1 \) and \( y \) is twice differentiable, Takens' Embedding Theorem guarantees that the map defined by mGRADE, \( \mathbf{u}_t \mapsto \mathbf{h}_t \), is generically an embedding of the original dynamic manifold \( \mathcal{M}\in \mathbb{R}^d\) into \( \mathbb{R}^m \) \citep{Takens1981}. 
That is, the mGRADE hidden state \( \mathbf{h}_t \) reconstructs the underlying system dynamics $\mathbf{u}_t$ up to a smooth, invertible change of coordinates (a diffeomorphism).
\end{proof}

\subsubsection{Training and Evaluation Details}\label{sec:appendix_lorenz_details}

\paragraph{Task Description}
We train mGRADE to perform autoregressive next-step prediction on the first dimension of 2000 randomly generated trajectories of 500 timesteps from the 3-dimensional chaotic
Lorenz attractor, using the non-standardized Lorenz flow from the dysts package \citep{gilpin2024dysts}. 
We add 5\% Gaussian time-independent observational noise to every trajectory. 
We test with randomly generated sequences of the same length as the training sequences, using either the first dimension or the 2 unobserved dimensions of the attractor (to evaluate generalization capabilities).

\paragraph{Hyperparameters}
We compare a single-layer \gls{mgrade} to a 2-layer \gls{mingru} with the hyperparameters outlined in \cref{tab:lorenz-hp}.
We choose the 2-layer \gls{mingru} as a comparison for fairness, given that a single-layer \gls{mingru} does not provide its update gate or candidate activations with any temporal information (they only depend on the current input). 
Neither model uses the \gls{mlp} or layer normalization.
For the learning rate, we use AdamW \citep{loshchilov2018decoupled-adamw} over all weights, and standard Adam \citep{kingma2017adam} for the biases, normalization layers, and \gls{dcls} positions.
We use a cosine annealing learning rate scheduler without warmup.
For the $\mathbf{z}$ gate biases, we use an "open" initialization for the (where the gate bias is set such that $\sigma (\mathbf{W}_z\mathbf{x}_t)\approx1$) for all models, while using the traditional zero initialization for all other biases.
For weights, we initialize with a truncated normal distribution, with a standard deviation set to $\sqrt{1/\text{fan\_in}}$. 
Other hyperparameters are outlined in \cref{tab:lorenz-hp}.
Reported results are averaged over 5 random seeds.

\begin{table}[h]
\caption{Hyperparameters used for the Dynamics Reconstruction Task. L: number of layers. D/H: model dimensionality and hidden state size per layer (no hidden state expansion for any of these models). K: kernel count. $\Gamma$: kernel length. LR: learning rate. B: batch size. Epochs: max epochs set for the run. WD: weight decay.}
\label{tab:lorenz-hp}
\vspace{0.2cm}
\begin{center}
\renewcommand{\arraystretch}{1.2}
\resizebox{0.7\textwidth}{!}{
\begin{tabular}{lccccccccc}
\toprule
Parameter   & L     & D/H     & K     & $\Gamma$    & LR    & B     & Epochs    & WD \\
\midrule
mGRADE     & 1     & 10    & 1     & 32        & 0.004 & 32    & 200       & 0.0 \\
minGRU        & 2     & 10    & --     & --         & 0.004 & 32    & 200       & 0.0 \\
\bottomrule
\end{tabular}
}
\end{center}
\end{table}

\paragraph{Loss Metric}\label{sec:appendix_mase} Following \citep{Ostrow2024}, we use \gls{mase} as the loss metric.
\gls{mase} compares the mean absolute
error made by the model across a sequence with the mean absolute error that would have been incurred had the default prediction been that the next state is equal to the current state at every timestep \citep{hyndmanMASE2006}. 
This can be expressed as follows,
\[
\text{MASE} = \frac{  \frac{1}{T} \sum_{i=1}^{T}\left| y_t -\hat{y}_t \right| }{\frac{1}{T-1} \sum_{i=2}^{T}\left| y_t -y_{t-1} \right|},
\]
where $\hat{y}$ is the model's prediction, $T$ is the sequence length, and $y_t$ is what is being predicted.
Notably, a \gls{mase} $>$ 1 means that naive forecasting (using the current $y_{t-1}$ to predict the next state $y_t$) works better than the forecasting model.
Thus, a model only offers meaningful predictive power if it can achieve a \gls{mase} $<$ 1.

\paragraph{Manifold Similarity Metric} The Nearest neighbor Overlap metric is used to evaluate how smoothly mGRADE maps the original Lorenz attractor manifold to the hidden state $h_t$.
It is calculated following \citet{Ostrow2024}.
For every point $i$ on every trajectory in the test set, we find the k-nearest neighbors, i.e., the set of time points $Q_{u}(i)$ where the trajectory gets closest to the selected point in the original 3 dimensions of the Lorenz system.
Then we map the selected point to the model's hidden state and evaluate how many of the k-nearest neighbors in the hidden state ($Q_{h}(i)$) are in fact the same k-nearest neighbors in the original system mapped into the hidden state.
The number of overlapping neighbors relative to the total number of neighbors evaluated then yields the Nearest neighbor Overlap metric when applied to every point $i$ on every trajectory in the dataset:
\[
\text{Overlap}(\textit{Original Manifold}, \textit{Hidden State}) = \frac{1}{n} \sum_{i=1}^{n} \frac{ \left| Q_u(i) \cap Q_h(i) \right| }{k}
\]
where $k$ is the number of neighbors evaluated (20 in this work), and $n$ is the total number of datapoints available across all trajectories.

\paragraph{Results}
The \gls{mase} over epochs for next-step prediction of the dimension that the models are trained on is shown for the \gls{mgrade} and \gls{mingru} models in \cref{fig:SI_lorentz_over_epoch}A.
\cref{fig:SI_lorentz_over_epoch}B shows the average \gls{mase} when performing next-step prediction over the 2 dimensions that the models do not observe during training.
\gls{mgrade} outperforms the purely recurrent models in both cases.
Notably, the \gls{mingru} model does not generalize well to unobserved dimensions resulting in a MASE $>$ 1.
\gls{mgrade} also achieves a higher Nearest neighbor Overlap (in \%) over epochs for both models in \cref{fig:SI_lorentz_over_epoch}C.

\begin{figure*}[ht!]
    \centering \includegraphics[width=\textwidth]{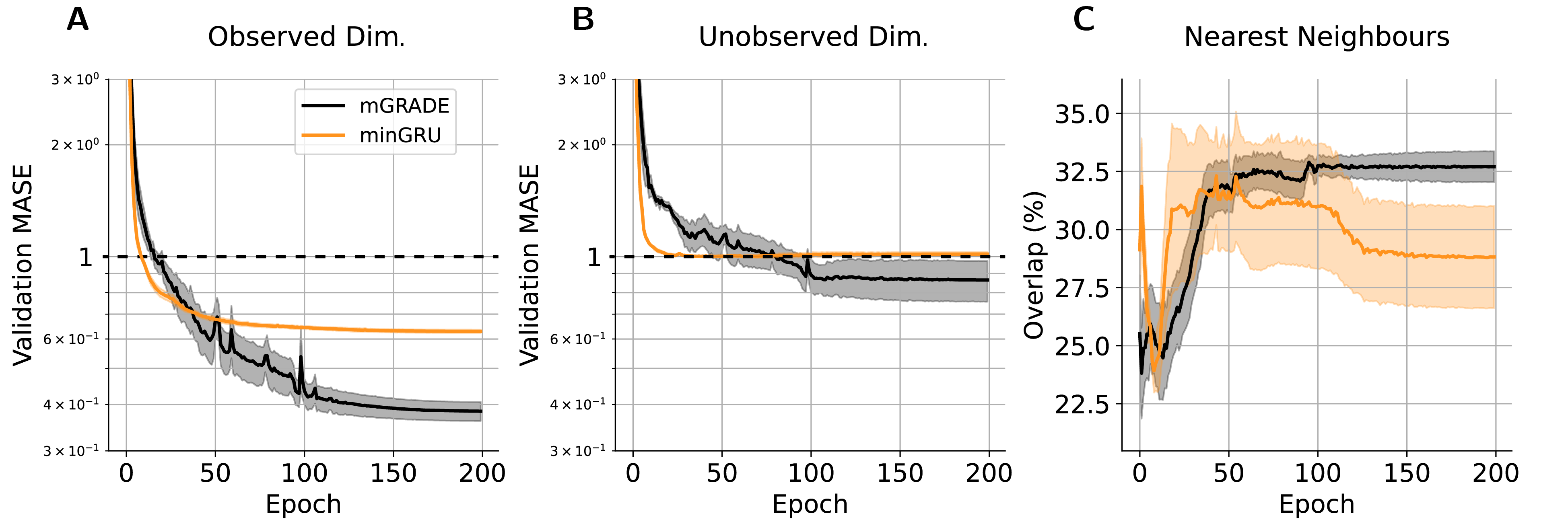}    
    \caption{\textbf{mGRADE achieves lower training loss and higher manifold overlap, even on unobserved dimensions of dynamics.}
    \textbf{A)} Validation MASE loss (mean ± stde) over training epochs on predicting the observed first dimension. \textbf{B)} Out-of-Distribution MASE loss (mean ± stde) over training epochs on predicting the dimensions unobserved during training (2 and 3 of the original attractor). Dotted black line marks MASE $=$ 1 which indicates no predictive power. \textbf{C)} Nearest neighbor Overlap (mean ± stde) of 20 nearest neighbors to each trajectory point between original state space and hidden state space over training epochs.}
    \label{fig:SI_lorentz_over_epoch}
\end{figure*}

\subsection{Hidden State Principal Components}\label{sec:appendix_lorenz}

To provide an additional visualization aid on the similarity of the various models' hidden state representations of the Lorenz attractor dynamics in \cref{sec:attractor}, we plot each of the top 3 Principal Components (PC) against each other in \cref{fig:SI_lorentz_2d_shapes} together with their corresponding explained variance (EV). Compare to \cref{fig:F2}B, C for the corresponding 3D plots.

\begin{figure*}[ht!]
    \centering \includegraphics[width=\textwidth]{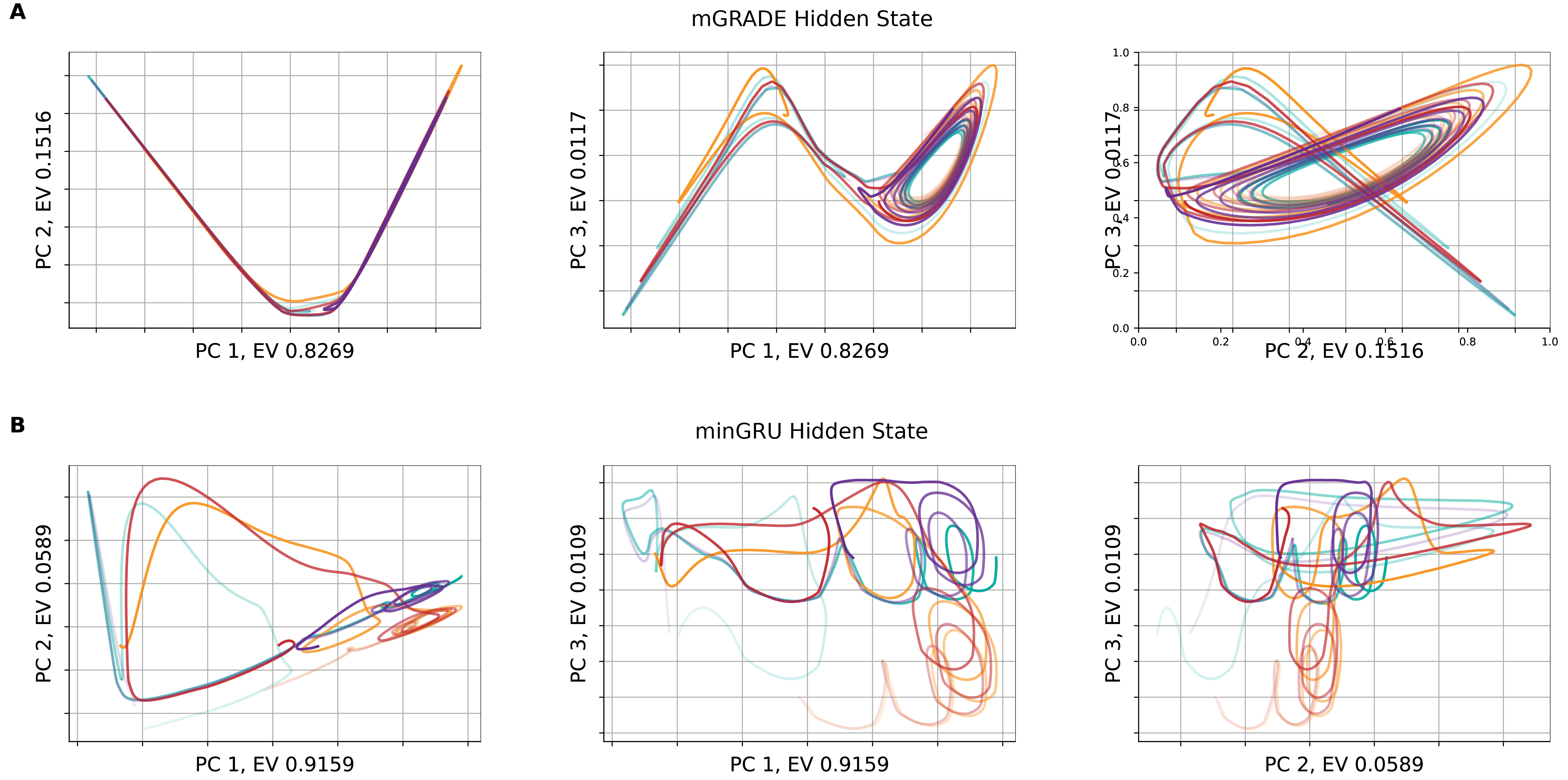}    
    \caption{\textbf{\gls{mgrade} reconstructs original dynamics in hidden state principal components.}
    \textbf{A)} Three top PCs of single-layer \gls{mgrade} plotted against each other with corresponding explained variance (EV). (see \cref{fig:F2}B for 3D plot).
    \textbf{B)} Three top PCs of 2-layer minGRU plotted against each other with corresponding explained variance (EV). (see \cref{fig:F2}C for 3D plot).}
    \label{fig:SI_lorentz_2d_shapes}
\end{figure*}

\subsection{High-frequency Pattern Recognition Task}\label{sec:appendix_highfreqtask}

\paragraph{Task Description} 
We train on sequences with a total length of $165$ timesteps, randomly generated at every training step.
Each sequence contains $5$ out of $16$ possible features.
Each feature is $l=32$ timesteps long and contains input symbols selected from an alphabet of size $n=16$.
After each marker symbol $\texttt{m}$ in the sequence, the goal is to output the associated class label $S_i$ of the preceding feature indexed by $i$.
The \textit{feature frequency} is adapted by changing how many times $r$ an input symbol is repeated within the feature before switching to a different input symbol (\Cref{fig:highfrequencytask}).
We normalize the feature frequency to the sampling rate so that it is equal to $1/r$.
Thus, a feature frequency of $1.0$ means that the input symbols in the pattern change every timestep, while a feature frequency of $0.5$ means that each input symbol is repeated 2 times before switching.
All models are trained on sequences that contain patterns with increasing normalized feature frequencies ($0.07$, $0.17$, $0.33$, $1.00$).

\begin{figure}[h!]
    \centering
    \includegraphics[width=0.65\textwidth]{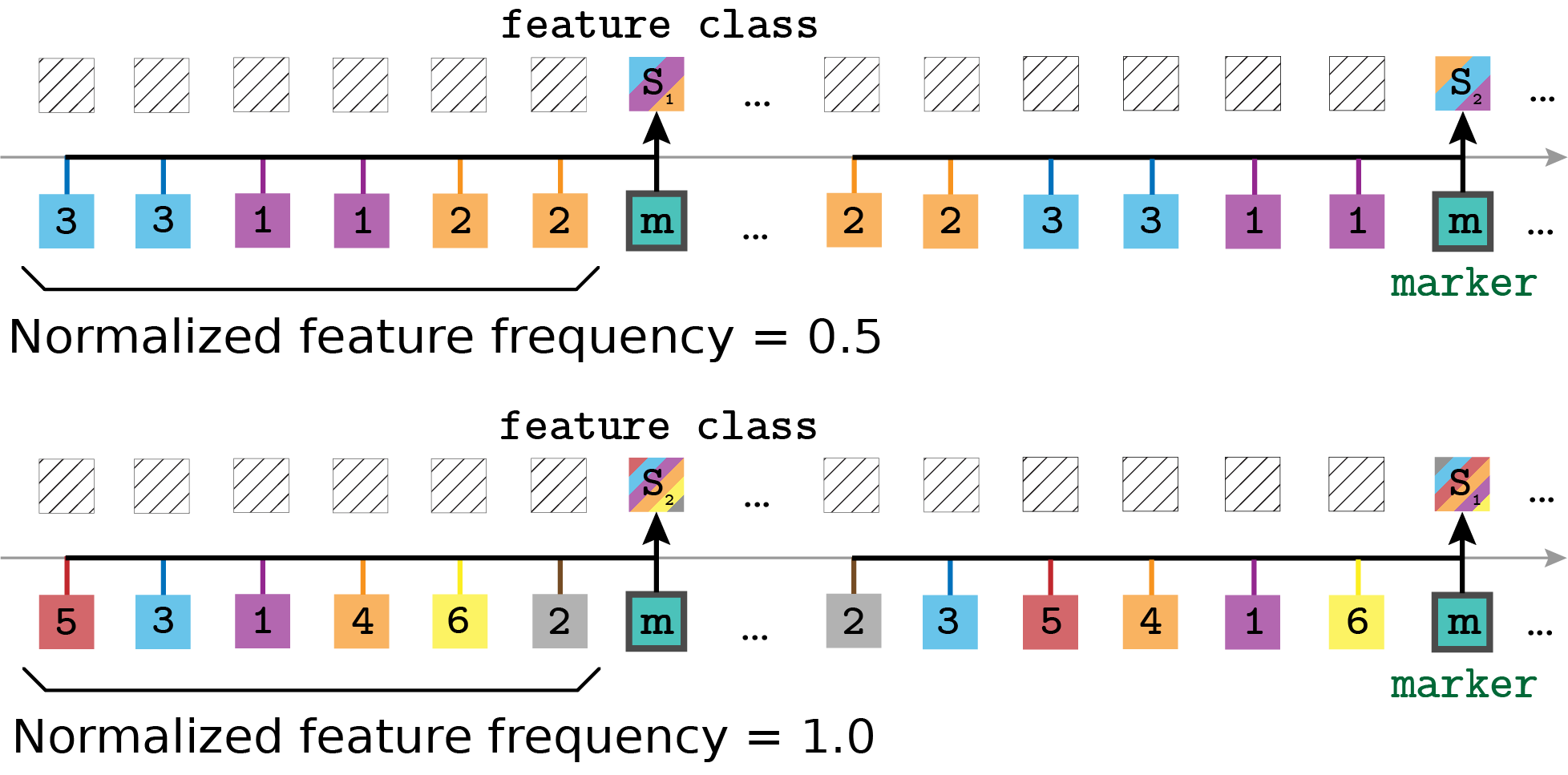}
    \caption{\textbf{High-frequency pattern recognition task.} The task requires classifying several features within a sequence, with each feature consisting of randomly ordered input symbols with different feature frequencies (top: low frequency, bottom: high frequency). After being presented a feature, the model should output the associated class. The number of times that the input symbol changes within a feature is the inverse of the feature frequency.}
    \label{fig:highfrequencytask}
\end{figure}

\paragraph{Hyperparameters}
We compare single- and 2-layer \glspl{mgrade} with single- and 2-layer \glspl{mingru}. 
The 2-layer models were chosen as comparisons to demonstrate that \gls{mgrade} can pass high-frequency feature information through multiple gated layers, while \gls{mingru} cannot.
None of the models use an \gls{mlp} or layer normalization between layers.
For the optimization, we use AdamW \citep{loshchilov2018decoupled-adamw} over weights, and standard Adam \citep{kingma2017adam} for the biases, normalization layers, and \gls{dcls} positions.
We use a cosine annealing learning rate scheduler without warmup.
For biases, we use a standard zero initialization, and for weights, we initialize with a truncated normal distribution, with a standard deviation set to $\sqrt{1/\text{fan\_in}}$. 
Other hyperparameters are outlined in \cref{tab:highfreq-hp}. Reported results are averaged over 5 random seeds.

\begin{table}[h]
\caption{Hyperparameters for the High Frequency Recognition Task. L: number of layers. D/H: model dimensionality and hidden state size per layer (no hidden state expansion for any of these models). K: kernel count. $\Gamma$: kernel length. LR: learning rate. B: batch size. Training Steps: number of batches presented during training. WD: weight decay.}
\label{tab:highfreq-hp}
\vspace{0.2cm}
\begin{center}
\renewcommand{\arraystretch}{1.2}
\resizebox{0.7\textwidth}{!}{
\begin{tabular}{lcccccccc}
\toprule
Parameter   & L     & D/H     & K     & $\Gamma$    & LR    & B     & Training Steps    & WD \\
\midrule
mGRADE     & 1     & 16    & 8     & 16        & 0.004 & 64    & 200       & 0.1 \\
mGRADE        & 2     & 16    & 8     & 16         & 0.004 & 64    & 200       & 0.1 \\
minGRU        & 1     & 20    & --     & --        & 0.004 & 64    & 200       & 0.1 \\
minGRU        & 2     & 16    & --     & --        & 0.004 & 64    & 200       & 0.1 \\

\bottomrule
\end{tabular}
}
\end{center}
\end{table}

\paragraph{Loss Metric}
Cross-entropy loss is used over the predicted feature classes, which are set for each timestep with a marker $\texttt{m}$ in the input.
Each model is tested on randomly generated test sequences using the same feature frequency as during training.

\paragraph{Results}
\cref{fig:SI_high_freq}A shows the classification validation accuracy over training steps for the models that were trained on a feature frequency of 0.33.
With increasing feature frequency, the test accuracy of the \gls{mgrade} models remains roughly constant (or even slightly increases) while the 2-layer \gls{mingru} collapses when trying to classify higher-frequency features (\cref{fig:SI_high_freq}B).
Notably, adding a second layer to the \gls{mingru} decreases its classification accuracy, implying that high-frequency information is lost over successive recurrent layers as suggested by previous literature on both gated \glspl{rnn} and \glspl{ssm} \citep{rahaman2019spectralbias, yu2025tuning}.
\gls{mgrade} on the other hand maintains performance even with 2 layers, indicating that the temporal convolution indeed preserves high-frequency information even through the gated recurrent unit.

\begin{figure*}[ht!]
    \centering \includegraphics[width=0.7\textwidth]{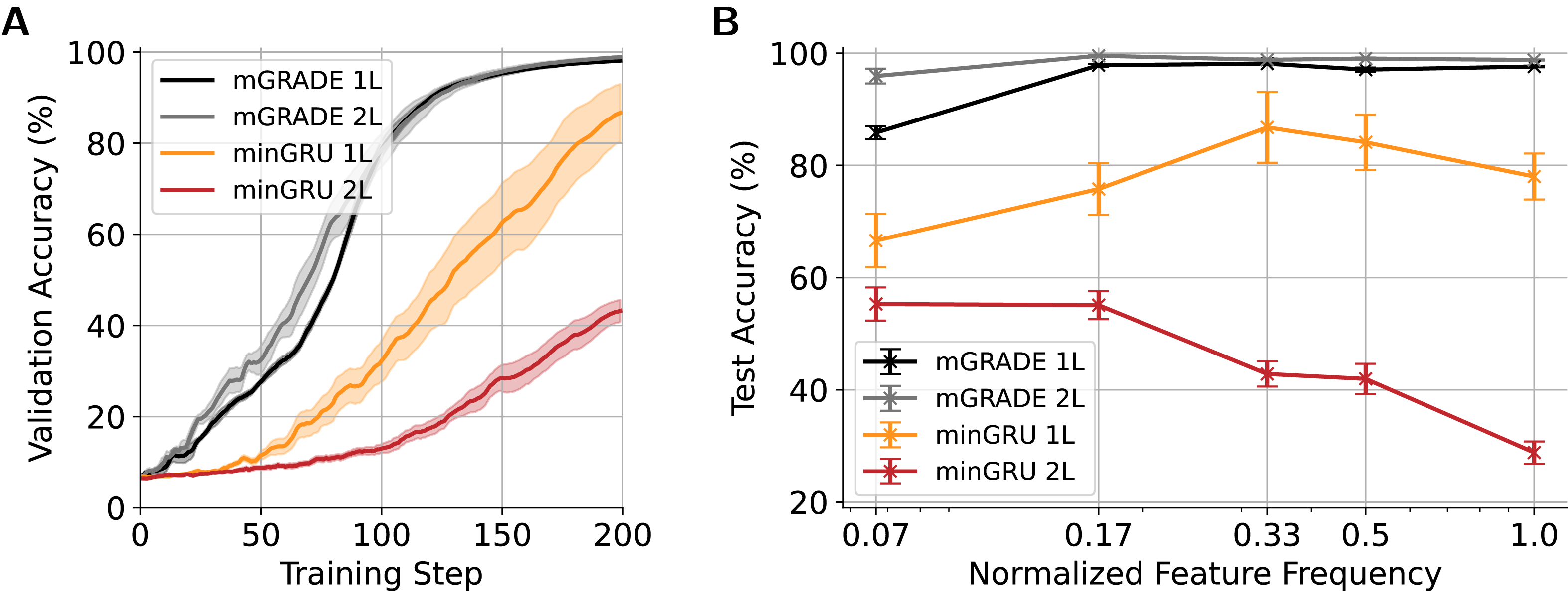}    
    \caption{\textbf{mGRADE recognizes features with high frequencies better than pure gated RNNs.}
    \textbf{A)} Validation accuracy (mean ± stde) for high-frequency pattern recognition task with a feature frequency of 0.33. \textbf{B)} Final test accuracy (mean ± stde) after training on different feature frequencies. 1L stands for single-layer and 2L for 2-layer models.}
    \label{fig:SI_high_freq}
\end{figure*}

\subsection{Flip-Flop Predictive Modeling Task}

\subsubsection{Proofs for Flip-Flop Predictive Modeling Capabilities}\label{sec:appendix_flipflop_proof}

Here we detail the full proofs associated with \cref{sec:flipflop}. For convenience, we start by reiterating the definition of a Flip-Flop language.

\textbf{Definition 1} (Flip-Flop Language).
    Let the alphabet be \( \Sigma = \texttt{\{w, r, i, 0, 1\}} \), where \texttt{\( w \)}, \texttt{\( r \)}, and \texttt{\( i \)} represent instruction symbols ("write", "read", "ignore"), and \texttt{\( 0 \), \( 1 \)} represent value symbols. 
    The \textit{Flip-Flop languages} \( L_{ff} \) consist of sets of strings over \( \Sigma \) that alternate between instructions and values (e.g., $\texttt{w 0 r 0 i 1}$), satisfying the condition that after every $\texttt{r}$ symbol, the subsequent value equals the value symbol following the most recent $\texttt{w}$. All valid strings begin with \(\texttt{w}\).

\textbf{Definition 2} (Predictive Modeling).
    For a string \( s \in L_{ff} \) and a prefix \( s[1:t] \) ending at position \( t \) with symbol \( a_t \), predictive modeling requires outputting the \textbf{prediction set} \( P_i \subseteq \Sigma \) of valid next symbols \( a_{t+1} \) such that \( s[1:t] \, a_{t+1} \) remains a prefix of some string in \( L_{ff} \). 
    We say that a model \textbf{predictively models} \( L_{ff} \) iff its output at a single timestep \(t\) encodes all the information needed such that a linear classifier can return the next prediction set with \(100\% \) accuracy.

\textbf{mGRADE}

{
\renewcommand{\thetheorem}{\ref{theorem:mgrade}}
\begin{theorem}[Flip-Flop Modeling with mGRADE]
  A single-layer mGRADE with at least 2 delays in the convolutional component can predictively model a Flip-Flop language, $L_{ff}$, at arbitrary length.
\end{theorem}
\addtocounter{theorem}{-1}
}

\begin{proof}
    We prove the above theorem by construction.
    For notation, we use to $\mathbf{c}=[\mathbf{a}, \mathbf{b}]$ to denote stacking column vectors $\mathbf{a}\in\mathbb{R}^A$, $\mathbf{b}\in\mathbb{R}^B$ into another column vector $\mathbf{c}\in\mathbb{R}^{A+B}$.
    In addition, we assume all symbols and possible prediction sets are one-hot encoded in the input and output, respectively.

    Consider a single-layer mGRADE with an input sequence of length $T$, \( \mathbf{u}_{1:T} \in \mathbb{R}^{|\Sigma|\times T}\), where \( \mathbf{u}_t \) is the one-hot encoded vector of the symbol at position \( t \) in the string, and a model dimensionality of $D = 2|\Sigma|$.
    To match the model dimensionality, pass $\mathbf{u}_{1:T}$ at every timestep through a simple linear projection to match the model dimensionality (as described in the model architecture).
    Set this linear projection to simply stack 2 copies of the input $\mathbf{u}_{1:T}$ in a single vector $\hat{\mathbf{u}}_{1:T}=[\mathbf{u}_{1:T},\mathbf{u}_{1:T}]$.
    Set the convolution component to have 2 different delays to the 2 copies of the input in $\hat{\mathbf{u}}_{1:T}$ at times $t$ and $t-1$ (delay of $0$ and $1$ respectively) with weights of 1.
    Set the interpolation function width $v$ narrow enough such that the convolution kernel elements are zero everywhere besides at $t$ and $t-1$.
    Given this kernel construction, the output of the temporal convolution at any timestep $\mathbf{x}_t\in \mathbb{R}^{2|\Sigma|}$ will depend only on the current input $\mathbf{u}_t, \mathbf{u}_{t-1}\in \mathbb{R}^{|\Sigma|}$.
    Specifically, given the 2 different delays to the input copies, $\mathbf{x}_t=[\mathbf{u}_t, \mathbf{u}_{t-1}]$.\\
    
    Set the hidden state size $H=D=2|\Sigma|$ such that \( \mathbf{h}_t \in \mathbb{R}^{2|\Sigma|} \).
    Split the hidden state into 2 components
    $\mathbf{h}_t = [\mathbf{h}_t^{\text{stored}}, \mathbf{h}_t^{\text{current}}]$
    where \( \mathbf{h}_t^{\text{stored}} \in \mathbb{R}^{|\Sigma|} \) is parameterized such that it stores the value following the most recent \( w \), and \( \mathbf{h}_t^{\text{current}} \in \mathbb{R}^{|\Sigma|} \) passes on the current input. 
    Define \( \mathbf{z}_t = [\mathbf{z}_t^{\text{stored}}, \mathbf{z}_t^{\text{current}}] \), and \( \tilde{\mathbf{h}}_t = [\tilde{\mathbf{h}}_t^{\text{stored}},\tilde{\mathbf{h}}_t^{\text{current}}] \) as the corresponding gate and candidate states.

    The mGRADE updates \( \mathbf{h}_t \) as follows.

    1. Compute \(\mathbf{z}_t^{\text{stored}}\):
    \[ \mathbf{z}_t^{\text{stored}} = \sigma(\mathbf{W}_z^{\text{stored}}\mathbf{x}_t)=\sigma(\mathbf{W}_z^{\text{stored}} [\mathbf{u}_t, \mathbf{u}_{t-1}])
    \]
    where \( \sigma \) is the sigmoid function, and \( \mathbf{W}_z^{\text{stored}} \in \mathbb{R}^{|\Sigma|\times 2|\Sigma|} \) is the weight matrix coupling \(\mathbf{z}_t^{\text{stored}}\) and \(\mathbf{x}_t=[ \mathbf{u}_t, \mathbf{u}_{t-1}]\). Set \( \mathbf{W}_z^{\text{stored}} \) such that the weight corresponding to the location of the 1 in the one-hot encoding of $\texttt{w}$ in \( \mathbf{u}_{t-1} \) approaches \( -\infty \), and weights for all other components to approach \( +\infty \). Thus:
    \[\mathbf{z}_t^{\text{stored}} =
    \begin{cases} 
        0 & \text{if } \mathbf{u}_{t-1} = \texttt{w}, \\
        1 & \text{otherwise}.
    \end{cases}
    \]

    2. Compute \(\tilde{\mathbf{h}}_t^{\text{stored}}\):
    \[\tilde{\mathbf{h}}_t^{\text{stored}} = \mathbf{W}_{\tilde{\mathbf{h}}}^{\text{stored}} \mathbf{x}_t = \mathbf{W}_{\tilde{\mathbf{h}}}^{\text{stored}} [\mathbf{u}_t, \mathbf{u}_{t-1}]\]
    where \( \mathbf{W}_{h}^{\text{stored}} \in \mathbb{R}^{|\Sigma|\times2|\Sigma|} \) is the weight matrix coupling \(\tilde{\mathbf{h}}_t^{\text{stored}}\) and \(\mathbf{x}_t=[\mathbf{u}_{t},\mathbf{u}_{t-1}]\). Set \( \mathbf{W}_{h}^{\text{stored}} \) as a block matrix containing the identity in the component multiplied with \( \mathbf{u}_{t} \) and zeros in the component multiplied with \( \mathbf{u}_{t-1} \) to the effect that the \( \mathbf{u}_{t} \) component gets passed on whereas \(\mathbf{u}_{t-1}\) does not. Thus:
    \[\tilde{\mathbf{h}}_t^{\text{stored}} = \mathbf{u}_t\]

    3. Update \(\mathbf{h}_t^{\text{stored}}\):
    \[
    \mathbf{h}_t^{\text{stored}} = \mathbf{z}_t^{\text{stored}} \odot \mathbf{h}_{t-1}^{\text{stored}} + (1 - \mathbf{z}_t^{\text{stored}}) \odot \tilde{\mathbf{h}}_t^{\text{stored}}
    \]
    When \( \mathbf{u}_{t-1} = \texttt{w} \), \( \mathbf{z}_t = 0 \) in the asymptotic limit, so \( \mathbf{h}_t^{\text{stored}} = \tilde{\mathbf{h}}_t^{\text{stored}} = \mathbf{u}_t \) (a value $\texttt{0}$ or $\texttt{1}$); otherwise, \( \mathbf{z}_t = 1 \), so \( \mathbf{h}_t^{\text{stored}} = \mathbf{h}_{t-1}^{\text{stored}} \).

    4. Update \(\mathbf{h}_t^{\text{current}}\). Let each component of \(\mathbf{W}_z^{\text{current}}\in \mathbb{R}^{|\Sigma|\times2|\Sigma|} \) approach \( -\infty \) such that \(\mathbf{z}_t^{\text{current}} \textit{}\approx 0 \) always. 
    As above, set \(\mathbf{W}_{\tilde{h}}^{\text{current}}\in \mathbb{R}^{|\Sigma|\times2|\Sigma|} \) such that \(\tilde{\mathbf{h}}_t^{\text{stored}} = \mathbf{u}_t\). 
    In the asymptotic limit, the update expression
    \[ \mathbf{h}_t^{\text{current}} = \mathbf{z}_t^{\text{current}} \odot \mathbf{h}_{t-1}^{\text{current}} + (1 - \mathbf{z}_t^{\text{current}}) \odot \tilde{\mathbf{h}}_t^{\text{current}}
    \]
    evaluates to \(\mathbf{h}_t^{\text{current}} = \mathbf{u}_t\).\\

    Thus, \( \mathbf{h}_t^{\text{stored}} \) retains the one-hot vector of the value following the most recent $\texttt{w}$, and \( \mathbf{h}_t^{\text{current}} \) passes on \( \mathbf{u}_t \).\\
 
    The possible prediction sets over \( L_{ff} \) are \( P_1 = \{0, 1\} \), \( P_2 = \{ \texttt{w, r, i} \} \), \( P_3 = \{ \texttt{0}\} \), \( P_4 = \{ \texttt{1}\} \).
    Given the structure of \( L_{ff} \), we can associate each set of possible input symbols to its corresponding prediction set.\\

    If \( \mathbf{u}_t \in \{ \texttt{0, 1} \} \) (a value), then \( \mathbf{u}_{t+1} \) is an instruction, so the output should be \(P_2 = \{ \texttt{w, r, i} \}\).
    If \( \mathbf{u}_t \in \{ \texttt{w, r, i} \} \) (an instruction), then \( \mathbf{u}_{t+1} \) must be a value. 
    Specifically, if \( \mathbf{u}_t =\texttt{w} \) or \( \mathbf{u}_t = \texttt{i} \), then the output should be \(P_1 = \{ \texttt{0, 1} \}\), since the following value is arbitrary. 
    On the other hand, if \( \mathbf{u}_t = \texttt{r} \), then \( \mathbf{u}_{t+1} \) must match the value after the most recent \( \texttt{w} \), which (by the preceding construction) is stored in \( \mathbf{h}_t^{\text{stored}} \). 
    Thus, the output at time $t$ should be
    \[\mathbf{h}_t^{\text{stored}}=
    \begin{cases} 
        p_3 & \text{if } \mathbf{h}_t^{\text{stored}} = \texttt{0}, \\
        p_4 & \text{otherwise}.
    \end{cases}
    \].

    A linear classifier, parameterized by a weight matrix \( \mathbf{W}_c \in \mathbb{R}^{4 \times 2|\Sigma|} \) and bias \( \mathbf{b}_c \in \mathbb{R}^4 \), maps \( \mathbf{h}_t \) to the correct prediction set \( P_i \):
    \[
    i = \text{argmax}(\mathbf{W}_c \mathbf{h}_t + \mathbf{b}_c),
    \]
    where \( i \) is the index corresponding to one of the four prediction sets. 
    Since \( \mathbf{h}_t = [\mathbf{h}_t^{\text{stored}}, \mathbf{u}_t] \) provides both the current symbol and the stored value, \( \mathbf{W}_c \) can be trained (or constructed) to distinguish these cases based on the one-hot encoded positions in \( \mathbf{h}_t^{\text{current}}=\mathbf{u}_t \) and \( \mathbf{h}_t^{\text{stored}} \). Specifically, we set \( \mathbf{W}_c \) to have a high weight between the one-hot encodings of \( \texttt{w} \) and \( \texttt{i} \) and the corresponding index of the prediction set \( \{\texttt{0,1} \}\). We also set a high weight between the one-hot encodings of \( \texttt{0} \) and \( \texttt{1} \) and the corresponding index of the prediction set \( \{\texttt{w, r, i} \}\). Finally, we set a negative bias term to the indices of the \( \{\texttt{0} \}\) and \( \{\texttt{1} \}\) prediction sets with a corresponding larger weight between \( \texttt{r} \) and those indices such that \( \texttt{r} \) is enough to activate \( \{\texttt{0} \}\) and almost enough to activate \( \{\texttt{1} \}\). If now the weight from \(\mathbf{h}_t^{\text{stored}}\) to \( \{\texttt{0} \}\) is negative and to \( \{\texttt{1} \}\) is positive, then a stored \(\texttt{1}\) will activate \( \{\texttt{1} \}\) (and suppress \( \{\texttt{0} \}\)) while a stored \(\texttt{0}\) will allow \( \{\texttt{0} \}\)) to be active while not adding to the activation of \( \{\texttt{1} \}\).\\

    To conclude, for any prefix \( s[1:t] \in L_{ff} \):

    1. \(\mathbf{h}_t^{\text{stored}} \) can correctly store the value symbol after the most recent \(\texttt{w}\) (at least in an asymptotic limit w.r.t. the weight magnitudes of the update gate \(\mathbf{z}_t^{\text{stored}}\)).

    2. \( \mathbf{h}_t^{\text{current}} = \mathbf{u}_t \) encodes the current input symbol.

    3. Given mGRADE's outputs \(\mathbf{h}_t\), a linear classifier can be constructed to output the correct prediction set classification as required by the language’s rules, handling arbitrary lengths since \( \mathbf{h}_t^{\text{stored}} \) persists across timesteps.

    For initial states or prefixes without \( \texttt{w} \), assume \( \mathbf{h}_0^{\text{stored}} = \mathbf{0} \), but since every \( \texttt{r} \) in a valid string follows a \( \texttt{w} \), \( \mathbf{h}_t^{\text{stored}} \) is always defined when needed. Thus, mGRADE can predictively model \( L_{ff} \) at arbitrary length.
\end{proof}

\textbf{Fixed-length Context Models}\label{sec:appendix_flipflop_tcn_proof}

\begin{theorem}[Flip-Flop Modeling with Fixed-length Context Models] A model with a fixed-length context window for a fixed memory size cannot predictively model a Flip-Flop language, $L_{ff}$, at arbitrary lengths.
\end{theorem}

\begin{proof} Consider a sequence of length $T_c + 3$, where $T_c$ is the context window given some fixed memory size. Start with $\texttt{w v}$, follow with $T_c$ $\texttt{i}$ instructions, and end with $\texttt{r}$. The correct prediction after $\texttt{r}$ is $\texttt{v}$, but $\texttt{v}$ lies outside the context window, forcing chance-level performance. Note that increasing the context length is the obvious solution to this problem however, the correspondingly increasing memory costs eventually become prohibitive for very long sequences.
\end{proof}

Note that models with fixed-length temporal contexts given some fixed memory size include Transformers and \glspl{tcn}.

\subsubsection{Training and Evaluation Details}\label{sec:appendix_flipflop_details}

\paragraph{Task Description}
We train on the Flip-Flop dataset from \citet{liu2023attentionglitch}, containing 1M valid Flip-Flop sequences of $512$ timesteps, where training data contains $\texttt{i}$ instruction symbols with probability $p(\texttt{i})=0.8$, such that the expected distance between any $\texttt{w}$ and $\texttt{r}$ symbol is 10 timesteps. 
This dataset slightly simplifies the full predictive modeling task to focus on recalling the correct value symbol after an $\texttt{r}$ as described in \cref{sec:flipflop}. 
For testing, we use out-of-distribution data with sparse $\texttt{w}$ and $\texttt{r}$ (expected distance around 100 timesteps) to stress long-range dependencies.

\paragraph{Hyperparameters} 
We compare a single-layer \gls{mgrade} to a 5-layer \gls{tcn} (with an exponentially increasing receptive field) and a single-layer \gls{lru} augmented by \gls{dcls}.
The \gls{lru} is used as a drop-in replacement for the gated recurrent component of \gls{mgrade}, demonstrating the importance of \gls{mgrade}'s update gate relative to a linear time-invariant \gls{ssm} like the \gls{lru}.
We use the code supplied in \citep{Zucchet2024minimalLRU} as the basis for our \gls{lru} implementation.
None of the models use layer normalization.
For the optimization, we use AdamW \citep{loshchilov2018decoupled-adamw} over the weights, and Adam \citep{kingma2017adam} for the biases, the normalization layers, and the \gls{dcls} positions.
We also use a cosine annealing learning rate scheduler without warmup.
For the $\mathbf{z}$ gate biases, we use a "closed" initialization (where the gate bias is set such that $\sigma (\mathbf{W}_z\mathbf{x}_t)\approx0$), while using the traditional zero initialization for all other biases.
For all weights, we initialize with a truncated normal distribution with a standard deviation set to $\sqrt{1/\text{fan\_in}}$. 
Other hyperparameters are outlined in \cref{tab:flipflop-hp}.
Reported results are averaged over 3 random seeds.

\begin{table}[h]
\caption{Hyperparameters for the Flip-Flop Modeling Task. L: number of layers. D/H: model dimensionality and hidden state size per layer (no hidden state expansion for any of these models). K: kernel count (for all layers). $\Gamma$: kernel length (per layer from input to output for the \gls{tcn}). LR: learning rate. B: batch size. Training Steps: number of batches presented during training. WD: weight decay.}
\label{tab:flipflop-hp}
\begin{center}
\renewcommand{\arraystretch}{1.2}
\begin{tabular}{lcccccccc}
\toprule
Parameter   & L     & D/H     & K     & $\Gamma$    & LR    & B     & Training Steps    & WD \\
\midrule
mGRADE     & 1     & 32    & 1     & 2        & 0.004 & 64    & 250,000       & 0.1 \\
LRU + DCLS        & 1     & 32    & 1     & 2         & 0.004 & 64    & 250,000       & 0.1 \\
TCN        & 5     & 32    & 16     & {\scriptsize16/32/64/128/256}        & 0.004 & 64    & 250,000       & 0.1 \\
\bottomrule
\end{tabular}
\end{center}
\end{table}

\paragraph{Loss Metric}
Cross-entropy loss is used as the training loss.
The reported accuracy is how often the model correctly recalls the value symbol after the most recent $\texttt{w}$ when encountering a $\texttt{r}$.
Since there are 2 possible value symbols, chance level performance lies at 50\%.

\paragraph{Results}
\cref{fig:SI_flipflop}A shows the recall accuracy of each model over training steps.

\begin{figure*}[ht!]
    \centering \includegraphics[width=0.78\textwidth]{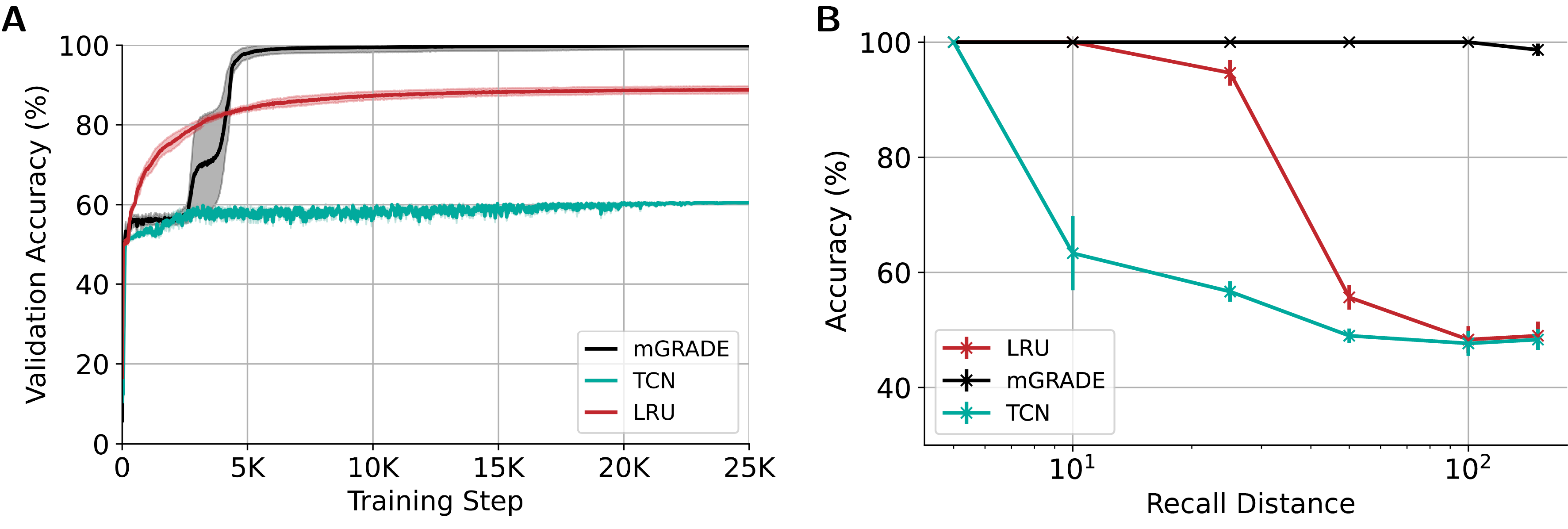}    
    \caption{\textbf{mGRADE solves Flip-Flop modeling task better than TCNs and non-gated RNNs.}
    \textbf{A)} Validation accuracy (mean ± stde) over training steps. \textbf{B)} Recall accuracy (mean ± stde) for different recall distances.}
    \label{fig:SI_flipflop}
\end{figure*}

\subsubsection{Recall Distance}\label{appendix:flipflop_recalldistance}
In addition to the results reported in \cref{sec:flipflop}, we evaluated how well the model recalls the most recent value after a $\texttt{w}$ given increasing distances between the $\texttt{w}$ and $\texttt{r}$ (the \textit{recall distance}).
For this, we construct Flip-Flop strings with one $\texttt{w}$ at the beginning and a $\texttt{r}$ in the middle, with different numbers of $\texttt{i}$ symbols (with corresponding value symbols) in between.
\cref{fig:SI_flipflop}B shows the recall accuracy of each of the models trained using the setup described above and tested on different recall distances.
Note that the only model consistently performing accurate recall over distances of up to 100 timesteps is \gls{mgrade}.
Even the \gls{lru} model decreases in accuracy with increasing recall distance, demonstrating the utility of a gated recurrent component.

\subsection{Selective Copying Task}\label{sec:appendix_selectivecopy}

\subsubsection{Training and Evaluation Details}

\paragraph{Task Description}
Following the training setup in \citet{gu2024mamba} and \citet{feng2025were}, we randomly generate sequences of $4096$ timesteps at each training step.
We train over $300,000$ steps.
Each sequence contains $16$ randomly distributed value symbols selected from an alphabet of size $n=16$. 
After seeing the $\texttt{m}$ symbol at the end of the sequence, the goal is to output the value symbols in the order they were received.
For testing, we generated new sequences of the same length.

\paragraph{Hyperparameters}
Following the architecture used in \citet{gu2024mamba}, we use a 2-layer \gls{mgrade}, comparing it to a 2-layer \gls{lru} augmented by \gls{dcls} and a 6-layer \gls{tcn} (with an exponentially increasing receptive field).
Just like the Flip-Flop modeling task, we use \citep{Zucchet2024minimalLRU} as the basis for a drop-in \gls{lru} replacement into the \gls{mgrade} architecture.
All models use encoders and decoders at the input and output, respectively.
For the optimization, we use AdamW \citep{loshchilov2018decoupled-adamw} for the weights, and Adam \citep{kingma2017adam} for the biases, the normalization layers, and the positions of the \gls{dcls}.
We also use a cosine annealing learning rate scheduler without warmup.
For the $\mathbf{z}$ gate biases, we use the \gls{ugi} initialization from \citet{gu2020improving-ugi}, while using the traditional zero initialization for all other biases.
We initialize the weights with a truncated normal distribution, with a standard deviation set to $\sqrt{1/\text{fan\_in}}$. 
Other hyperparameters are outlined in \cref{tab:selectivecopy-hp}.
Reported results are averaged over 2 seeds (because of the compute-intensive nature of this task).

\begin{table}[h]
\caption{Hyperparameters for the Selective Copying Task. L: number of layers. D/H: model dimensionality and hidden state size per layer (no hidden state expansion for any of these models). K: kernel count. $\Gamma$: kernel length (per layer from input to output for the \gls{tcn}). LR: learning rate. B: batch size. Training Steps: number of batches presented during training. WD: weight decay.}
\label{tab:selectivecopy-hp}
\begin{center}
\renewcommand{\arraystretch}{1.2}
\begin{tabular}{lcccccccc}
\toprule
Parameter   & L     & D/H     & K     & $\Gamma$    & LR    & B     & Training Steps    & WD \\
\midrule
mGRADE     & 2     & 64    & 32     & 128        & 0.001 & 64    & 300,000       & 0.1 \\
LRU + DCLS        & 2     & 64    & 32     & 128         & 0.001 & 64    & 300,000       & 0.1 \\
TCN        & 6     & 100    & 128     & {\scriptsize 128/256/512/1024/2048/4096}        & 0.003 & 64    & 300.000       & 0.0 \\
\bottomrule
\end{tabular}
\end{center}
\end{table}

\paragraph{Loss Metrics}
Cross-entropy loss is used over the final outputs after the $\texttt{m}$ symbols. 
The final accuracy is evaluated on randomly generated test sequences of the same length.

\paragraph{Results} \cref{fig:SI_selectivecopy} shows the accuracy over training steps for each model.

\begin{figure*}[ht!]
    \centering \includegraphics[width=0.5\textwidth]{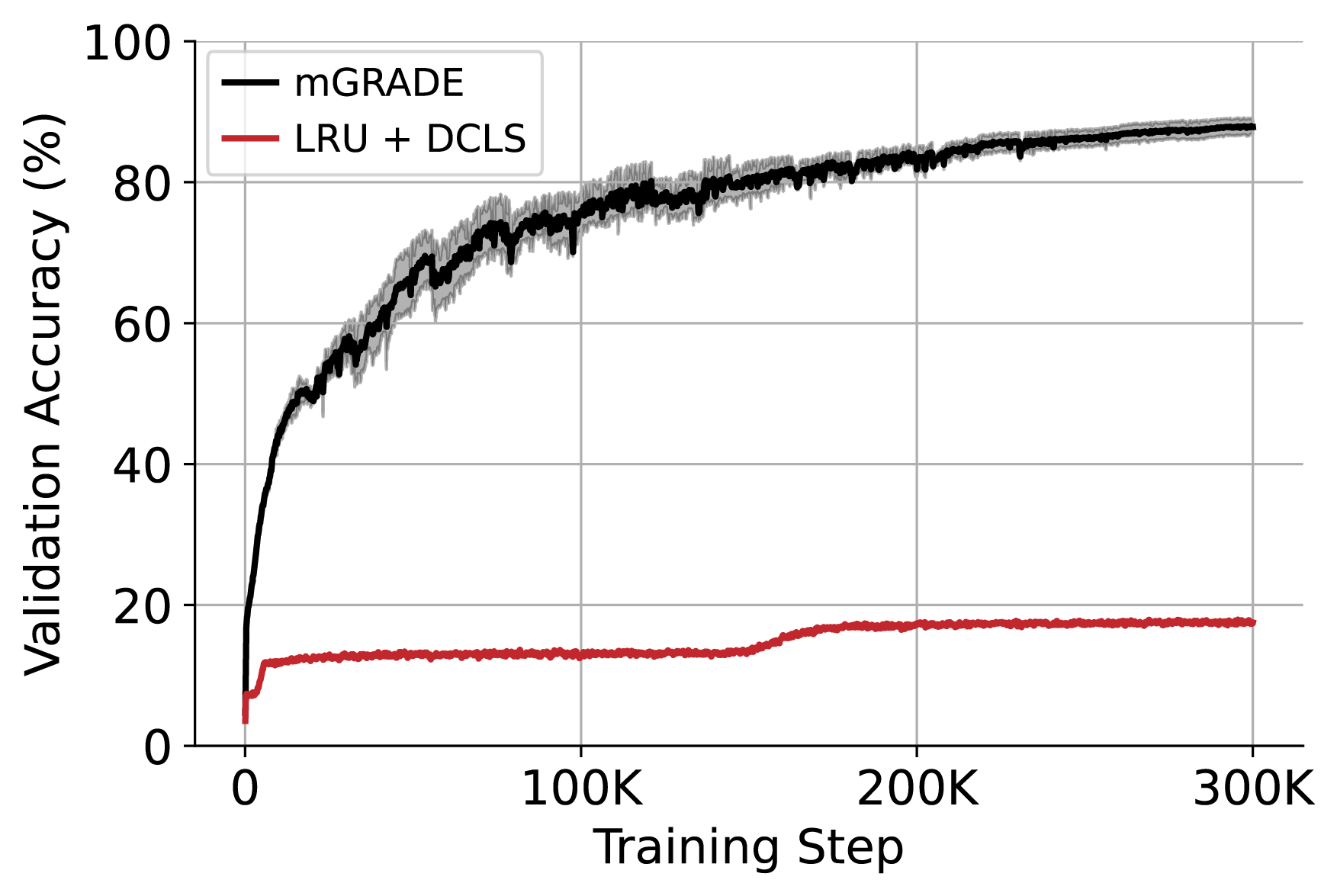}    
    \caption{\textbf{mGRADE solves the Selective Copying Task better than LRU.}
    Validation accuracy (mean ± stde) over training steps.}
    \label{fig:SI_selectivecopy}
\end{figure*}

\FloatBarrier

\section{\gls{lra} and \gls{gsc} setup}\label{sec:appendix_hyperparams}

\subsection{Hyperparameters}

In \cref{tab:lra-hp}, we provide the hyperparameters used for the \gls{mgrade} results on \gls{lra} and \gls{gsc} reported in \cref{tab:lraresults} and \cref{tab:sc35_compact}. 
We use the encoder and decoder at the input and output, respectively, as well as an \gls{mlp} and layer normalization in each \gls{mgrade} layer.
For the optimization, we use AdamW \citep{loshchilov2018decoupled-adamw} for the weights, and Adam \citep{kingma2017adam} for the biases, normalization layers, and \gls{dcls} positions.
In addition, the learning rate was scaled by 5 for the \gls{dcls} positions.
For all tasks, we use a cosine annealing with linear warmup learning rate scheduler.
We use two initialization schemes for the $\mathbf{z}$ gate biases, the traditional zero initialization and the \gls{ugi} from \citet{gu2020improving-ugi}.
We selected the zero-initialization for all the other biases. 
We use a truncated normal distribution for all weights (except the ones of the temporal convolution block), with a standard deviation set to $\sqrt{1/\text{fan\_in}}$. 
For the temporal convolution block, we set $\sqrt{\alpha/K}$ where $\alpha$ is a scaling hyperparameter.
We use gradient clipping for every task except for Pathfinder, with a threshold of 10. 
For Pathfinder we use gradient global normalization with a threshold at 2. 
We did not use dropout. 
Finally, for the Image task, we introduced an extra linear layer at the output of the \gls{mingru} component and before the addition with the corresponding skip connection.

In \cref{tab:lra-lra-hp}, we provide the hyperparameters used for the results of our implementation of a fully causal and parameter-optimized \gls{lru} on \gls{lra} reported in \cref{tab:lraresults}.
We use the code supplied in \citep{Zucchet2024minimalLRU} without modifications, except for the Retrieval task which requires a specialized decoder.
For this task, we use the specialized decoder for the Retrieval task from the official repository for \citep{smith2023simplified} as a drop-in replacement for the default decoder used in \citep{Zucchet2024minimalLRU}.
As in \cite{orvieto2023resurrecting-lru}, we use a cosine annealing with linear warmup learning rate scheduler as well as AdamW \citep{loshchilov2018decoupled-adamw} for weight optimization and Adam \citep{kingma2017adam} for biases.
For all internal recurrent parameters, a smaller learning rate was used, determined by a factor $\beta$ of the original learning rate.
We also used the ring initialization from \cite{orvieto2023resurrecting-lru}, ensuring that the transition matrix eigenvalues were initialized between $\text{r}_{\text{min}}$ and $\text{r}_{\text{max}}$.
For all these \gls{lru}-specific hyperparameters, we used those reported in the \cite{orvieto2023resurrecting-lru}, modifying only the hidden and model dimensionality to match our \gls{mgrade} implementation in terms of memory footprint.
The learning rates were optimized using a base 10 logarithmic grid search.
For our causal \gls{lru}, we only reported the highest validation accuracy (57.4\%) reached given that our implementation was not able to significantly exceed chance level (50\%) in terms of test accuracy, despite sweeping the hidden and model dimensionalities (reaching model sizes of up to 2M), the ring initialization (using $\text{r}_{\text{min}}=0.9$ and $\text{r}_{\text{max}}=0.999$ as well as $\text{r}_{\text{max}}=0.9999$), the base learning rates (from 0.1 to 0.00001), and training for more than 800 epochs.
Note that the original \gls{lru} experiments only report results on Pathfinder using a bidirectional model \citep{orvieto2023resurrecting-lru}.

Finally note that \gls{lra} also includes the PathX task, which we do not evaluate for either \gls{mgrade} or our \gls{lru} implementation.
The models reported in Table 6 that solve PathfinderX all use (acausal) bidirectional processing.
Since mGRADE is designed for causal, streamed processing, this experiment was outside of our scope.
Nevertheless, \gls{mgrade}'s promising results on the \gls{gsc} dataset (which uses a sequence length that is as long as that of PathX, 16K) hint at \gls{mgrade}'s potential suitability even to this task.

\begin{table}[h]
\caption{Hyperparameters used for the reported \gls{mgrade} results on LRA. L: number of layers. D/H: model dimensionality and hidden state size per layer (no hidden state expansion for any of these models). K: kernel count. $\Gamma$: kernel length. LR: learning rate. B: batch size. Epochs: max epochs set for the run. WD: weight decay. ZBI: $\mathbf{z}$ gate bias initialization. $\alpha$: scaling factor for \gls{dcls} weight initialization. WU: number of epochs for the learning rate linear warmup.}
\label{tab:lra-hp}
\begin{center}
\renewcommand{\arraystretch}{1.2}
\begin{tabular}{lccccccccccc}
\toprule
Parameter   & L     & D/H   & K     & $\Gamma$  & LR     & B     & Epochs    & WD     & ZBI           & $\alpha$  & WU    \\
\midrule
ListOps     & 6     & 32    & 2     & 16        & 0.003  & 64    & 100       & 0.1    & \gls{ugi}     & 0.05      & 10     \\
Text        & 6     & 32    & 2     & 8         & 0.002  & 32    & 100       & 0.1    & zero          & 0.25      & 10    \\
Retrieval   & 3     & 64    & 2     & 8         & 0.003  & 32    & 20        & 0.1    & \gls{ugi}     & 0.05      & 4    \\
Image       & 6     & 128   & 8     & 256       & 0.004  & 64    & 100       & 0.1    & zero          & 0.1       & 10    \\
Pathfinder  & 6     & 128   & 8     & 256       & 0.003  & 64    & 100       & 0.02   & zero          & 1         & 10    \\
\midrule
\gls{gsc}   & 6     & 64    & 16    & 64        & 0.0005 & 8     & 40        & 0.05   & \gls{ugi}     & 0.25      & 4    \\
\bottomrule
\end{tabular}
\end{center}
\end{table}

\begin{table}[h]
\caption{Hyperparameters used for the results of our \gls{lru} reproduction on LRA. L: number of layers. D/H: model dimensionality and hidden state size per layer. $\text{r}_{\text{min}}$/$\text{r}_{\text{max}}$: minimum and maximum eigenvalue initialization radii. LR: learning rate. $\beta$: scaling factor for transition matrix learning rate. B: batch size. Epochs: max epochs set for the run. WD: weight decay. WU: number of epochs for the learning rate linear warmup.}
\label{tab:lra-lra-hp}
\begin{center}
\renewcommand{\arraystretch}{1.2}
\begin{tabular}{lcccccccccc}
\toprule
Parameter   & L     & D/H   & $\text{r}_{\text{min}}$/$\text{r}_{\text{max}}$     & LR  & $\beta$    & B     & Epochs    & WD  & WU    \\
\midrule
ListOps     & 6     & 33/32    & 0.0/0.99     & 0.001 & 0.5 & 32    & 100       & 0.05   & 10     \\
Text        & 6     & 34/32    & 0.5/0.9     & 0.001 & 0.1 & 32    & 100       & 0.05     & 10    \\
Retrieval       & 6     & 48/64    & 0.5/0.9     & 0.001 & 0.5 & 64    & 25       & 0.05   & 4     \\
Image       & 6     & 160/256    & 0.9/0.999     & 0.01 & 0.25 & 64    & 100       & 0.1   & 10     \\
\bottomrule
\end{tabular}
\end{center}
\end{table}

\subsection{Activity buffer memory footprint}\label{sec:appendix_activ_buff_calc}

In this section, we explain how we compute the total buffer memory used by the baseline models in \cref{tab:lraresults} and \cref{tab:lramemory}.
S4 \citep{gu2022efficiently}, DSS variants  \citep{gupta2022diagonal}, Liquid-S4  \citep{hasani2023liquid} all implement a similar architecture where H single-input, single-output SSM heads of size N are used in parallel. 
Thus, the amount of memory used for all recurrent hidden states is given by the formula $L \times (H \times N)$.
S4-LegS \citep{gu2022efficiently} uses H bi-directional SSM heads of size N in parallel. 
Thus, the total number of states is given by the formula $L \times (H \times 2N)$.
S5 \citep{smith2023simplified}, LRU \citep{orvieto2023resurrecting-lru} use only a single head multi-input, multi-output SSM of size N. 
Thus, the total number of states is given by the formula $L \times N$.
HGRN \citep{qin2023HGRN} uses a similar architecture to \gls{mgrade}'s gated recurrent block, extended with complex states. 
Thus, the total number of states is given by the formula $L \times (2H)$.
The long-convolution models, SGConv \citep{li2023what-sgconv} and MRConv \citep{cunningham2024reparameterized}, need to buffer the activity of each neuron at each timestep, thus the total number of states is given by the formula $L \times H \times T$ (where $T$ is the sequence length).
To convert the number of states to the activation buffer memory footprint, we multiply by the precision used for the state representation and divide by 8.
For all models mentioned here, we confirmed that their code implementations use a default precision of 32 bits per activation.

\subsection{Parameter memory footprint for real-time processing on edge devices}\label{sec:appendix_embedded_memory}

In this section, we explain how real-time processing on edge devices impose taking into consideration some aspects that do not apply when running inference on GPUs. 
When processing inputs in real-time, it is too memory-expensive to save the entire sequence before processing it, let alone buffering the activities of each neuron for the entire sequence length. 
For this reason, besides \gls{mgrade}, only S4 \citep{gu2022efficiently}, DSS \citep{gupta2022diagonal}, Liquid-S4 \citep{hasani2023liquid}, S5 \citep{smith2023simplified}, and LRU \citep{orvieto2023resurrecting-lru} could be deployed on an embedded device and we report the adjusted numbers for these architectures in  \cref{tab:lramemory}. 
We leave out bi-directional architectures as we care for causal processing. 
We also leave out the convolution architectures too as \cref{tab:lraresults} shows that their activity buffers memory footprint is already above 1M for all tasks.
\newline
In S4 \citep{gu2022efficiently} and Liquid-S4 \citep{hasani2023liquid}, the recurrent matrix is parametrized as a \gls{dplr} matrix $ A = \Lambda - PP^*$. This means that it is parametrized by two vectors of dimension $N$ (state dimension). 
However, when running in step-by-step recurrent mode on an embedded device, A would need to be instantiated into a full $N \times N$ matrix, which increases the number of effective parameters substantially.
For example on the Image task, the number of parameters of S4 \citep{gu2022efficiently} increases from 12.98M to 59.54M (i.e a factor of 4.6).
Similarly to deploy \gls{mgrade} on an embedded device, we would need to fully materialize the \gls{dcls} kernels into vectors of dimension $\Gamma$. 
This increases the number of parameters of \gls{mgrade} from 712K to 896K on the Image task (i.e, a small 25\% increase).
With these results, we confirm that \gls{mgrade} is the architecture with the smallest memory footprint.

\begin{table}[h]
    \caption{\textbf{Fully Instantiated Memory Footprint for Recurrent Embedded Deployment on \gls{lra}.} Compare to \cref{tab:lraresults}. Note that S5 and HGRN use bidirectional input processing. Parameter (``Params.".) and activation memory (``Buff.") in bytes.}
    \label{tab:lramemory}
    \resizebox{\textwidth}{!}{  
    \vspace{0.2cm}
    \centering
    \begin{tabular}{lccccc}
    \toprule
             & {ListOps} & {Text} & {Retrieval} & {Image} & {Pathfinder} \\ 
        Model & Params. / Buff.\ Act. & Params. / Buff.\ Act. & Params. / Buff.\ Act. & Params. / Buff.\ Act. & Params. / Buff.\ Act. \\ 
    \midrule%
        S4~\citep{gu2022efficiently}                     & 12.60M / 191.4K         & 4.58M / 62.5K          & 27.86M / 382.8K         & 59.54M / 769.5K         & 26.72M / 382.8K  \\ 
        DSS$_\text{SOFTMAX}$ ~\citep{gupta2022diagonal}  & 804.7K / 191.4K         & 593.8K / 62.5K         & 3.39M / 382.8K         & 7.63M / 769.5K          & 2.29M / 382.8K \\ 
        Liquid-S4~\citep{hasani2023liquid}               & 1.42M / 31.2K           & 710.9K / 15.6K         & 29.01M / 382.8K         & 3.03G / 6.11M          & 27.86M / 382.8K \\ 
        S5~\citep{smith2023simplified}                   & 742.2K / 5.9K           & 4.96M / 4.3K           & 2.94M / 5.9K            & 19.47M / 9.0K         & 4.20M / 5.9K  \\ 
        LRU~\citep{orvieto2023resurrecting-lru}\tnote{3} & 742.2K / 5.9K           & 4.96M / 4.3K           & 2.94M / 5.9K            & $-$ / $-$           & 4.20M / 5.9K \\  
        HGRN~\citep{qin2023HGRN}             & 328.1K / 1.6K         & 3.35M / 3.9K        & 449.2K / 1.2K        & 78.63M / 23.8K         & 4.96M / 5.9K  \\ 
    \midrule
        {\gls{mgrade}}                                   & 164.1K / 11.7K           & 175.8K / 5.9K         & 410.2K / 6.6K        & 3.42M / 769.5K          & 3.04M / 769.5K  \\ 
    \bottomrule  
    \end{tabular}
    }
\end{table}

\subsection{Memory constraints of representative edge hardware for sequence modeling}\label{sec:appendix_mcu}
In computer architecture, there is a trade-off between memory density and access latency/energy cost. 
For example, a processor’s on-chip cache is very fast and energy-efficient to access but it is not dense and thus expensive to manufacture and limited in memory capacity. 
To mitigate this, memory is typically organised in a hierarchy from small and fast (on-chip cache) to large but slow (off-chip main memory or even SSDs and hard drives) \citep{mutlu2025memorycentric}. 
Achieving the latency and power required for energy-efficient real-time advanced sequence modeling \citep{covi2021edgecompute} requires working within the on-chip cache — leading us to the challenge of designing models small enough to fit in this cache. 
\gls{mgrade} matches the typical size constraints of on-chip memory in edge hardware.

We can illustrate \gls{mgrade}’s memory advantage for edge sequence modeling using the STM32N6 microcontroller already mentioned in \cref{fig:F1} \citep{stm32n6_2025}. 
STM32N6 has a flexible neural processing unit (NPU) along with a single-precision floating point unit (FPU), enabling it to store parameters and process activations with 32 bit precision with \gls{sota} latency and energy efficiency as long as the total model memory footprint fits inside its 4.2MB on-chip cache. 
The only model that fits this memory budget even for the most memory-intensive \gls{lra} tasks is \gls{mgrade} with its average 1.2MB and maximum 3.5MB memory footpring (see \cref{tab:lraresults}); the other models are forced to use the slow and energy-costly off-chip memory, limiting their utility for edge processing applications.

To highlight the fact that \gls{mgrade} still fits on edge platforms even for the most memory-intensive tasks evaluated in this work, \cref{fig:SI_memory_footpring_image} shows the average test accuracy across \gls{lra} versus the \textit{maximum }memory footprint (as opposed to the \textit{average} memory footpring shown in \cref{fig:F1}C).
In addition to the on-chip memory of STM32N6, we also plot the on-chip memory of several other edge-compatible hardware platforms.
We use the reported memory footprint from \cref{tab:lraresults} rather than the actual instantiated footprint shown in \cref{tab:lramemory}.

First of all, note that STM32N6's on-chip memory budget of 4.2MB is fairly typical of small edge-compatible hardware.
For example, the SiMa MLSoC has around 4MB on-chip memory \citep{simamlsoc_2024}, the Syntiant NDP250 offers 6MB \citep{syntiantndp250_2024}, and the Google Coral Edge TPU accelerator even reaches 8MB \citep{googlecoral_2026}, while the lower end comes out to 2MB with the Versal AI Edge Gen 2 SoC \citep{versaledgeai_2026}.

\cref{fig:SI_memory_footpring_image} shows clearly why \gls{mgrade}'s memory efficiency is so useful. 
Of all the other \gls{sota} models, only $\text{DSS}_{\text{SOFTMAX}}$ comes close to fitting on Google Coral while still underperforming \gls{mgrade} by $>$3\%.
Meanwhile, the only platform for which \gls{mgrade} is still too large is the Versal AI Edge Series Gen 2 with 2MB.
Even here, \gls{mgrade}'s average memory footprint of 1.5MB (see \cref{fig:F1}C) would comfortably fit.
Of course, quantization techniques may reduce the memory footprint required for the other \gls{sota} models; however, \glspl{ssm} are difficult to quantize, particularly their transition matrices \citep{zhao2025quantizingsmallscalestatespacemodels}.
The core components of \gls{mgrade}, delay-based convolutions and gated recurrence, have both been adopted in heavily quantized embedded settings before \citep{dagostino2024denram, blanken2025chameleon, billaudelle2025minimalist}.

\begin{figure*}[ht!]
    \centering \includegraphics[width=0.8\textwidth]{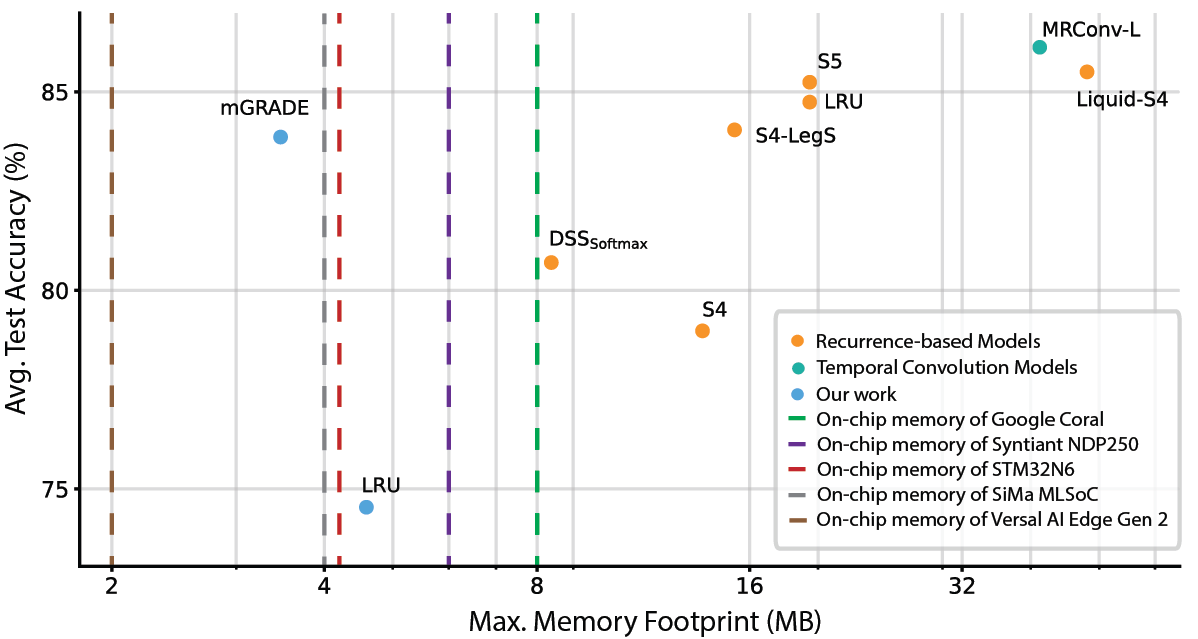}    
    \caption{\textbf{mGRADE performs the best within the memory constraints of representative embedded systems.} Average test accuracy across all \gls{lra} tasks in \cref{tab:lraresults} over the \textit{maximum} memory footprint required (compare with \cref{fig:F1}C which shows average test accuracy over \textit{average} memory footprint)}
    \label{fig:SI_memory_footpring_image}
\end{figure*}

\section{Evaluation on Neuromorphic Spiking Dataset}\label{sec:appendix_spiking}

To rigorously evaluate \gls{mgrade}'s claims of parameter and memory efficiency in strictly low-resource environments, we extended our benchmarking to the neuromorphic domain. While Spiking Neural Networks (SNNs) and specialized event-based architectures are often cited as the standard for low-power temporal processing, we hypothesized that \gls{mgrade}’s hybrid approach of combining lightweight gated recurrence with learnable delay convolutions could offer superior efficiency without requiring the specialized hardware or complex training dynamics often associated with SNNs. We selected the Spiking Heidelberg Digits (SHD) dataset, a standard neuromorphic benchmark consisting of approximately 10k recordings of spoken digits (0–9) in both English and German, resulting in 20 classes \citep{cramer2020heidelberg}. The dataset is constructed using an artificial cochlear model which results in 700 input channels. SHD is highly relevant for this evaluation because it requires the detection of precise temporal patterns within sparse spike trains to achieve good classification accuracy.

\subsection{Experimental Setup}\label{sec:appendix_spiking_expsetup}
We follow the same data preprocessing pipeline established in the DCLS-delays work \citep{hammouamri2024learning}, employing spatio-temporal binning to reduce input dimensionality. Input neurons were reduced from 700 to 140 by spatially binning over every 5 neurons. For the temporal dimension, we utilized zero right-padding to ensure a fixed window length for all sequences in a batch. Crucially, regarding the temporal resolution, the original DCLS-delays work reported using a discrete time-step of $\Delta t = 10$ ms. However, when using their official implementation with this reported configuration, we were unable to reproduce their state-of-the-art results. Through our own ablation of the baseline's configuration, we discovered that a finer temporal resolution of $\Delta t = 5$ ms was necessary to achieve performance parity with their reported figures. To ensure a fair and direct comparison, we therefore adopted this empirically verified configuration ($\Delta t = 5$ ms) for both the baseline reproduction and our \gls{mgrade} model. Consequently, to maintain the equivalent temporal receptive field (250ms), the sequence length (buffer size) was increased to 50 steps. We trained an \gls{mgrade} model using the aforementioned preprocessing (with 5 ms bins) with two recurrent layers of 64 hidden units each, operating over a sequence length (buffer size) of 50 steps. Crucially, while the DCLS-delays baseline relies on a single delay element per synapse (1 kernel count), our \gls{mgrade} configuration utilizes 10 learnable delay taps within the convolution window. This design allows \gls{mgrade} to sample the input history more densely while maintaining a compact recurrent state. We applied a dropout rate of 10\% for regularization. Finally, the output predictions are generated via a summation of the logit outputs across the entire sequence length, ensuring reproducibility with the DCLS-delays model. All reported results for both \gls{mgrade} and the reproduced DCLS-delays baseline are averaged over 5 random seeds.

\begin{table}[h]
    \centering
    \begin{threeparttable}
    \caption{\textbf{Classification Accuracy and Parameter Efficiency on the Spiking Heidelberg Digits (SHD) Dataset.} We compare \gls{mgrade} against state-of-the-art efficient spiking architectures. Parameter (``Params.'') and activation memory (``Buff.'') in bytes.}
    \label{tab:shd_results}
    \begin{tabular}{lccc}
    \toprule
        Model & Params. & Buff. & Test Accuracy \\
    \midrule
        SpikCommander\tnote{1} ~\citep{wang2025spikcommander} (1L-8-128) & 760K & - & \textbf{96.41\%} \\
        EventSSM\tnote{1} ~\citep{schone2024scalable} (6L-64) & 1.6M & - & 95.90\% \\
        DCLS\tnote{2} ~\citep{hammouamri2024learning} (2L-256) & 800K & 200 & 93.95\% $\pm$ 0.72 \\
    \midrule
        {\gls{mgrade}} (2L-64) & \textbf{257.6K} & 200 & 93.77\% $\pm$ 0.23 \\
    \bottomrule
    \end{tabular}
    \begin{tablenotes}
        \item[1] Results extracted from official publication.
        \item[2] Our reproduction.
    \end{tablenotes}
    \end{threeparttable}
\end{table}

\subsection{Results}\label{sec:appendix_spiking_results}
The results, summarized in \cref{tab:shd_results}, demonstrate that \gls{mgrade} is highly competitive with specialized spiking architectures while being significantly more parameter-efficient. \gls{mgrade} achieves an accuracy of 93.77\%, which is statistically comparable to the reproduced DCLS-delays model (93.95\%), yet it does so with 3x fewer parameters (64.4K vs. 200K). Furthermore, compared to other high-performing models like EventSSM (400K params), \gls{mgrade} requires roughly 6x fewer parameters to achieve competitive performance.  This evaluation successfully demonstrates that \gls{mgrade} establishes a new standard for parameter efficiency among this group while delivering competitive accuracy. It confirms that the combination of learnable delay embeddings (DCLS) and minimal gating (minGRU) is a powerful and highly efficient approach for sequence modeling in memory-constrained environments.

\FloatBarrier
\section{mGRADE Analysis}\label{sec:appendix_analysis}

\subsection{Ablation Study} \label{sec:appendix_ablation}
In section \cref{sec:theoretical}, we formally motivated the need for the temporal convolution and the gated recurrent component of \gls{mgrade} to tackle long-range dependency tasks. 
\cref{tab:lraablation} compares the performance of \gls{mgrade} to architectures using only recurrent or convolutional components (using a mean pooling operator).
Note that all architectures use the best hyperparameters found for mGRADE, scaling the hidden dimensionality to match parameter size.
This means that kernel lengths and kernel counts were matched across mGRADE's convolution component and the purely convolution-based architectures for each task.
The pure convolution-based models are the \gls{tcn}, consisting of stacked causal dilated temporal convolution layers, and the \gls{dcls} model, made up of stacked causal temporal \gls{dcls} layers, which can be thought of as \gls{mgrade} without the recurrent component.
For the pure gated recurrent architecture, we simply remove the convolutional component from the \gls{mgrade} layers, leaving us with the \gls{mingru} \citep{feng2025were}.
We focus on the ListOps, Image, and Pathfinder tasks from \gls{lra} as \citep{orvieto2023resurrecting-lru} already showed that pure linear \glspl{rnn} could solve the Text and Retrieval tasks to around 89\%. 
Both the \gls{tcn} and the \gls{dcls} models achieve good performance on Image while falling short on ListOps.
On the other hand, \gls{mingru} achieves a better result than \gls{mgrade} on ListOps, while performing poorly on the Image task. 
Besides \gls{mgrade}, none of these architectures learn on Pathfinder.
Our results are consistent with prior work demonstrating that purely convolution-based models (with around 10M parameters) achieve up to 83.62\% on Image but only up to 72.28\% and 52.93\% on Pathfinder and ListOps, respectively \citep{miralles2025lralocality}.
Overall, our ablation study validates the theoretical motivations for each component of an \gls{mgrade} layer and showcases the synergy between the temporal convolution and gated recurrent components.

\begin{table}[h]
\caption{\textbf{Ablation of \gls{mgrade}'s component on the \gls{lra} benchmark.} All architectures use the best hyperparameters found for \gls{mgrade}, we only scaled the layer width H to match the number of parameters. Parameter (``Params.".) and activation memory (``Buff.") in bytes.}
\label{tab:lraablation}
    \resizebox{\textwidth}{!}{   
    \vspace{0.2cm}
    \centering
    \begin{tabular}{lcccccc}
    \toprule
         & \multicolumn{2}{c}{ListOps} & \multicolumn{2}{c}{Image} & \multicolumn{2}{c}{Pathfinder} \\ 
        \cmidrule(lr){2-3}\cmidrule(lr){4-5}\cmidrule(lr){6-7}
        Model & Acc & Params. / Buff. & Acc & Params. / Buff. & Acc & Params. / Buff. \\ 
    \midrule
        TCN                           & 39.6 & 180K / 16K            & 85.3  & 2.9M / 880K         & $\times$ & 3.5M / 920K  \\ 
        DCLS                          & 43.8 & 164K / 12K            & 86.2  & 2.1M / 880K         & $\times$ & 2.1M / 960K  \\ 
        minGRU                        & \textbf{62.5} & 160K / 768   & 66.0  & 2.8M / 3K         & $\times$ & 2.5M / 3K  \\ 
        \gls{mgrade}                  & 61.9 & 160K / 12K            & \textbf{87.1} & 2.8M / 788K & \textbf{94.9} & 2.5M / 788K  \\ 
    \bottomrule
    \end{tabular}
    }
\end{table}

\subsection{Learned Kernels} \label{sec:appendix_trained_kernels}

We analyze the learned positions in the \gls{dcls} kernel across layers to understand how \gls{mgrade} adapts its temporal convolution mechanism to different task structures.
\Cref{fig:DCLS_ker_pos} shows the distribution of learned delay positions across the kernel (x-axis) and across hidden channels (y-axis) for models trained on the \gls{lra} Image, Pathfinder, and ListOps tasks (as done in section \cref{sec:appendix_ablation}).

\begin{figure*}[ht!]
    \centering \includegraphics[width=\textwidth]{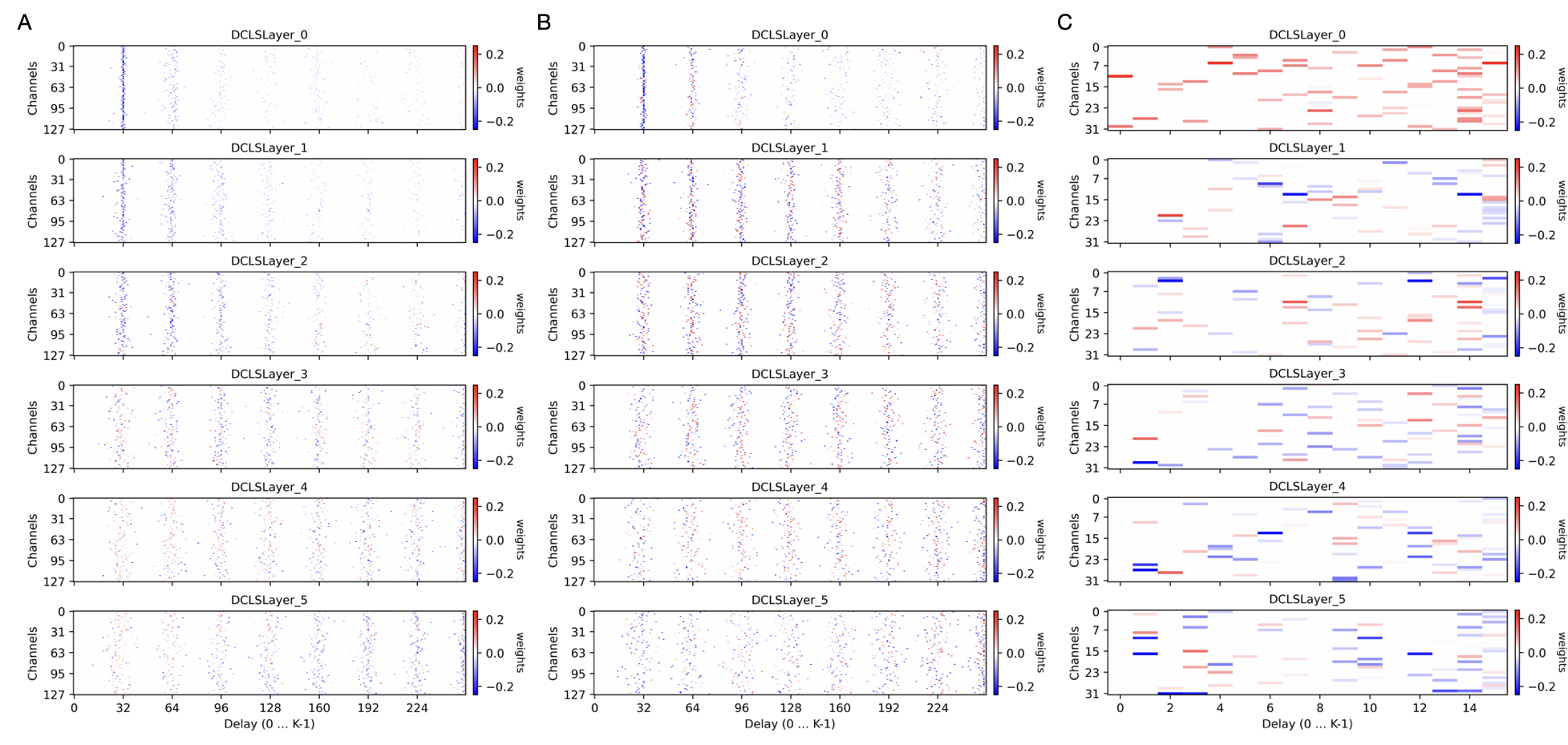}   
    \caption{\textbf{Learned \gls{dcls} delay positions across layers for \gls{lra} benchmark.} Each panel shows the distribution of learned delay positions (x-axis) across hidden channels (y-axis) for all layers (0-5, top to bottom) of models trained on (A) Image, (B) PathFinder, and (C) ListOps. Color intensity indicates weight magnitude (red: positive, blue: negative).
    \textbf{A)} Image (kernel size 256, 8 elements, 128 hidden dims): Positions cluster around vertical bands at positions 32, 64, and 96, which for a 32 $\times$ 32 image correspond to the pixels in the same column but one, two, and three rows above respectively, demonstrating learned preference for local structure. Deeper layers show increasing dispersion around the vertical bands (offsets ranging from 0-10) while maintaining locality. 
    \textbf{B)} Pathfinder (kernel size 256, 8 elements, 128 hidden dims): Delays disperse much more across the full 0-256 range, reflecting the sparse, non-local nature of path detection where relevant features (blobs and connecting paths) appear at arbitrary locations. Moderate vertical banding suggests some channels specialize for identifying the blobs at early layers. 
    \textbf{C)} ListOps (kernel size 16, 2 elements, 32 hidden dims): Full utilization of the 0-15 delay range across layers. The sparse sampling (2 of 16 positions per channel) enforces efficient information aggregation for bracket matching and operator precedence.
    }
    \label{fig:DCLS_ker_pos}
\end{figure*}

\subsubsection{LRA Image: Emergence of Spatial Local Processing on Sequential Images}
For the \gls{lra} Image task (\cref{fig:DCLS_ker_pos}A), the learned positions reveal a hierarchical local feature extraction pattern. 
In the early layers (0-2), positions sharply cluster around a delay of 32 timesteps. 
When processing the 32 $\times$ 32 images used in \gls{lra} sequentially, this delay precisely corresponds to the pixel directly above the current input pixel.
This concentration on immediate spatial neighbors (whether above or below) resembles the local receptive fields of classic 2D Convolutional Neural Networks (CNN), suggesting that \gls{mgrade}, through training, automatically tends towards spatial locality for image processing despite the sequential presentation of the image.
Deeper layers (3-5) maintain this locality bias, but with increased dispersion in time, effectively expanding the receptive field and with it the spatial context (while still remaining local).
The distinct vertical bands observed across channels indicate specialized feature detectors, some channels consistently attend to immediate neighbors while others look slightly further with a few specific offsets.
This hints towards the model using its learnable positions to capture irregularly spaced patterns and a larger range of spatial frequencies within the spatially local receptive field it builds over the image.

\subsubsection{LRA Pathfinder: Distributed Search for Sparse Structure}
In contrast, the delays learned by the model trained on the \gls{lra} Pathfinder task (\cref{fig:DCLS_ker_pos}B) exhibit slightly different patterns that reflect the task's more sparse, less local structure.
While still exhibiting some vertical clustering, delays appear more dispersed around the clusters compared to the Image task.
Particularly deeper layers show significantly weaker clustering than the corresponding layers in the model trained on Image. 
This distribution suggests that the model cannot rely only on local patterns, as the relevant features (two dots and their connecting path in a 32 $\times$ 32 image) can appear at arbitrary pixel locations across samples.
Early layers show moderate clustering at certain delay values visible as a few strong vertical bands, especially in layers 1 and 2. 
Later layers maintain broader delay coverage, suggesting they aggregate evidence across multiple spatial scales to determine path connectivity. 
The weaker locality bias in the positions indicates that spatially adjacent pixels in the Pathfinder images provide less predictive power than is the case for the natural images in Image.

\subsubsection{LRA Listops: Adaptation to Nested Dependencies}
The learned positions for ListOps (\cref{fig:DCLS_ker_pos}C) reveal a sophisticated strategy for parsing symbolic structures, which stands in contrast to the patterns seen in the Image task. A key feature is how \gls{mgrade}'s temporal receptive field spans the entire kernel length, all 16 timesteps across all layers. 
This reflects the task nature, where meaningful dependencies between operators, operands, and matching brackets occur at variable distances but remain bounded to a range of around 10 timesteps on average \citep{nangia2018listops}.
In the initial layers (0-1), the model establishes a strong inductive bias by concentrating weights at diverse positions, combining more local information (delays close around 3) with information at the boundaries of the kernel (positions around 14).
In contrast, the deeper layers exhibit a more distributed pattern, with channels dedicating kernel elements sparsely to specific and non-local delays. 
This sparse distributed sampling allows the model to track multiple long- and mid-range dependencies simultaneously and efficiently.

These contrasting solutions demonstrate \gls{mgrade}'s ability to adapt its information aggregation mechanisms to task structure without overpowering inductive biases (given that the kernel length is large enough). 
For structured data with strong spatially local correlations (Image), the model converges to classic 2D CNN-like local processing. 
For tasks requiring global reasoning over sparse features (Pathfinder), it maintains broader temporal coverage. 
Listops, on the contrary, requires a hybrid of mid-range and long-range aggregation in deeper layers.
This adaptive behavior arises from enabling the delays to be learnable, validating \gls{dcls} as a flexible alternative to \glspl{tcn} with fixed dilation rates, and further justifying our decision to use it in \gls{mgrade}.

\end{document}